\documentclass[runningheads]{llncs}
\usepackage{graphicx}
\usepackage{latexsym}
\usepackage{hyperref}
\usepackage{url}
\usepackage{amsmath,amssymb,amsfonts}
\usepackage{siunitx}
\usepackage{bm}
\usepackage{stackengine}

\newcommand{\monzeroun}{0\scalebox{0.75}[0.7]{-}\hspace{-1.5pt}1}
\begin{document}
\title{PAC-Bayesian Domain Adaptation Bounds for Multi-view learning}
%\author{First Author\inst{1}\orcidID{0000-1111-2222-3333} \and
%Second Author\inst{2,3}\orcidID{1111-2222-3333-4444} \and
%Third Author\inst{3}\orcidID{2222--3333-4444-5555}}
\author{Mehdi Hennequin \inst{1,2}\orcidID{0000-0001-8074-6520} \and Khalid Benabdeslem \inst{2} \and Haytham Elghazel\inst{2}}
\authorrunning{M. Hennequin et al.}
\institute{Galil\'e Group, 28 Bd de la République, 71100 Chalon-sur-Saône, France
\email{hennequin.pro.mehdi@gmail.com}\\
 \and Universit\'e Lyon 1, LIRIS,\\ UMR CNRS 5205, F-69622, France\\
\email{\{mehdi.hennequin,khalid.benabdeslem,haytham.elghazel\}@univ-lyon1.fr}}

\maketitle              % typeset the header of the contribution
\begin{abstract}
This paper presents a series of new results for domain adaptation in the multi-view learning
setting. The incorporation of multiple views in the domain adaptation was paid little attention in the previous studies. In this way, we propose an analysis of generalization bounds with Pac-Bayesian theory to consolidate the two paradigms, which are currently treated separately. Firstly, building on previous work by Germain et al. \cite{GER13,GER15}, we adapt the distance between distribution proposed by Germain et al. for domain adaptation with the concept of multi-view learning. Thus, we introduce a novel distance that is tailored for the multi-view domain adaptation setting. Then, we give Pac-Bayesian bounds for estimating the introduced divergence. Finally, we compare the different new bounds with the previous studies.  

\end{abstract}
\section{Introduction}
Predictive models must fit data with different distributions. In fact, in the theory of statistical learning, the strong hypothesis that training and test data are to be drawn from the same probability distribution. However, this assumption is often too restrictive to be used in practice or in many real-life applications. Indeed, a hypothesis is learned and deployed in different and significantly changing environments. Due to that, we obtain a shift in the data distributions. A typical solution for addressing this issue is to retrain the models. Nonetheless, this process of retraining can result in both time and financial expenses. Thus, we need to design methods for adapting a model from learning (source) data to test (target) data. In the field of machine learning, this scenario is commonly known as \textit{domain adaptation (DA)} or \textit{covariate shift} \cite{MOR12}. Essentially, DA techniques aim to address the challenge of learning when the learning task is the same but the domains exhibit variations in their feature spaces or marginal conditional probabilities.

On the other hand, data can be expressed through multiple independent feature sets, as stated in \cite{XU13}. As a result, the data can be partitioned into independent groups, known as views \cite{XU13}. In domain adaptation, these views are usually merged into a single view to align with the learning objective. Nonetheless, this process of merging can potentially result in negative transfer, as highlighted in \cite{ZHA20}, whereby the integration of each view's unique statistical characteristics could lead to the introduction of unwanted data or knowledge. We find little research on \textit{multi-view domain adaptation (MVDA)} \cite{YAN12,HEN22,MUN21} where considerable attention has been paid to algorithm, while analysis of generalization bound remains largely understudied. In this way, we propose a theoretical analysis by means of Pac-Bayesian theory, with the aim of unifying the two paradigms that have conventionally been treated as separate entities \cite{GER13,GER15,GOYAL17}. 

\subsubsection{Contribution.}
In this paper, we explore the PAC-Bayesian framework to propose MVDA bounds in a binary classification situation without target labels (unsupervised domain adaptation). To the best of the authors’ knowledge, no generalization bounds was proposed for \textit{unsupervised multi-view domain adaptation}. First, we propose a novel distance adapted in our setting to measure the distance between distribution. In a second phase, we introduce a general Pac-Bayesian theorem as proposed in \cite{GER09,GER15a,BEG16} for estimating the novel distance introduced. Then, we provide specialization of our general theorem to the three versions of PAC-Bayesian theorems \cite{MCA99b,SEE02,LAN05,CAT07} following the same principles as Germain et al. \cite{GER09,GER15,BEG16,GOYAL17}. Finally, we propose a PAC–Bayesian domain adaptation bound in the multi-view setting.

\section{Related Works}\label{related works}
In this section, we present theoretical studies of multi-view learning and domain adaptation related to the Pac-Bayesian theory. First, we introduce the theoretical concepts needed for the following sections. In a second phase, we recall the general PAC–Bayesian generalization bounds in the setting of binary classification. Then, we present the work done by on the PAC–Bayesian domain adaptation. Finally, we present the works achieved to Pac-Bayesian theory and multi-view learning.  
\subsection{Notation and Assumptions}\label{notation}

This section introduces the definitions and concepts needed for the following sections. Let $\mathcal{X} \in \mathbb{R}^{d}$ and $\mathcal{Y} = \{-1, +1\}$,  denote respectively input space of dimension $d$ and output space. In the \textit{scenario of unsupervised multi-view domain adaptation (UMVDA)} we consider, the learner receives two samples: a labeled sample from a source
domain $\mathbb{D}_{S}$, defined by a distribution $Q$ over $\mathcal{X} \times \mathcal{Y}$; and unlabelled  sample according to the target domain $\mathbb{D}_{T}$, defined by a distribution $P$ over $\mathcal{X} \times \mathcal{Y}$; $Q_\mathcal{X}, P_\mathcal{X}$ being the respective marginal distributions over $\mathcal{X}$. We denote by $\mathbb{S} = \{(\bm{x}_{i},y_{i})\}_{i=1}^{m} \in (\mathcal{X} \times \mathcal{Y})^{m}$ the labeled sample of size $m$ received from the source domain,  which is drawn i.i.d. from $Q$. In addition, we consider that the data instances can be represented or partitioned in $V$ different views. 
 More formally, for $v \in \{1,\ldots,V\}$, and $V \geq 2$ is the number of views of not-necessarily the same dimension. Throughout the paper, we will abbreviate $v \in \{1,\ldots,V\}$ with $v \in [\![\mathcal{V}]\!]$. The labeled samples in a multi-view setting are defined as follows: $\forall \; v \in [\![\mathcal{V}]\!], \; \mathbb{S} = \{(\bm{x}_{i}^{v},y_{i})\}_{i=1}^{m} \in (\mathcal{X}^{v} \times \mathcal{Y})^{m}$, with $\{\bm{x}_{i}^{v}\}_{i=1}^{m}$ assumed to be drawn i.i.d. according to distribution $Q$. Note that the multi-view observations $\{\bm{x}_{i}\}_{i=1}^{m} = \{(\bm{x}_{i}^{1},\dots,\bm{x}_{i}^{V})\}_{i=1}^{m}$ belong to a multi-view input set $\mathcal{X} = \mathcal{X}^{1} \times \dots \times \mathcal{X}^{V}$, with each $\mathcal{X}^{v}$ being a subset of $\mathbb{R}^{d_{v}}$, and $d_{v}$ denoting the dimension of the $v^{th}$ view. The total dimension $d$ is given by $d = d_{1} \times \dots \times d_{V}$. Similarly, we define unlabeled samples from the target domain as $\forall \; v \in [\![\mathcal{V}]\!], \mathbb{T}_{\mathcal{X}}  = \{\bm{x}'^{v}_{i}\}_{i=1}^{n}$ of size $n$, drawn i.i.d. according to $P_{\mathcal{X}}$ (note that $\mathbb{T} = \{(\bm{x}'^{v}_{i},y_{i})\}_{i=1}^{n}$ would be drawn i.i.d. according to $P$). In our context, we consider that we have no labels in the target domain, but we have prior knowledge about the views in both domains.

We define a hypothesis class $\mathcal{H}$ of hypotheses $h : \mathcal{X} \rightarrow \mathcal{Y}$ . Besides, for the concept of multi-view learning, we consider for each view $v \in [\![\mathcal{V}]\!]$, a set $\mathcal{H}_{v}$ of hypothesis $h : \mathcal{X}_{v} \rightarrow \mathcal{Y}$. The expected source risk or true source risk of $h \in \mathcal{H}$ over the distribution $Q$, are the probability that h errs on the entire source domain, $\mathcal{R}_{Q}(h) = \mathbb{E}_{(\bm{x},y)\sim Q}\big[\mathcal{L}_{\monzeroun} \big(h(\bm{x}),y\big)\big]$, where $\mathcal{L}_{\monzeroun}(a, b) = \mathbb{I}\big[a\neq b\big]$ is the $\monzeroun$-loss function and where $\mathbb{I}[a\neq b]$ is the indicator function which returns $1$ if $a\neq b$ and $0$ otherwise. The extension of the $\monzeroun$-loss to real-valued voters (of the form $f : \mathcal{X} \rightarrow [-1, 1]$) is given by the following definition, $\mathcal{L}_{\monzeroun}\big(f(x),y\big) = \mathbb{I}\big[y.f(x) \neq 0]$.
For any two functions $(h,h') \in \mathcal{H}$, we denote by $\mathcal{R}_{Q_{\mathcal{X}}}(h,h')$ the expected disagreement of $h(x)$ and $h'(x)$, which measures the probability that $h$ and $h'$ do not agree on the entire marginal distributions $Q_{\mathcal{X}}$ over $\mathcal{X}$, $\mathcal{R}_{Q_{\mathcal{X}}}(h,h') = \mathbb{E}_{\bm{x}\sim Q_{\mathcal{X}}} \big[\mathcal{L}_{\monzeroun} \big(h(\bm{x}),h'(\bm{x})\big)\big]$. The empirical source risk $\mathcal{R}_{\mathbb{S}}(h)$ for a given hypothesis $h \in \mathcal{H}$ and
a training sample $\mathbb{S} = \{(x_{i}, y_{i})\}^{m}_{i=1}$ where each example is drawn i.i.d. from $Q$ is defined as, $\mathcal{R}_{\mathbb{S}}(h) = \frac{1}{m}\sum_{i=1}^{m}\mathcal{L}_{\monzeroun} \big(h(\bm{x}_{i}),y_{i}\big)$. In the same way, we define the empirical source disagreement by $\mathcal{R}_{\mathbb{S}_{\mathcal{X}}}(h,h') =  \frac{1}{m} \sum_{i=1}^{m}\mathcal{L}_{\monzeroun} \big(h(\bm{x}_{i}),h'(\bm{x}_{i})\big)$,  where $\mathbb{S}_{\mathcal{X}} = \{(x_{i})\}^{m}_{i=1}$ where each example is drawn i.i.d. from $Q_{\mathcal{X}}$ . The expected target risk $R_{P}(\cdot)$ over $P$, the expected target disagreement $R_{P_{\mathcal{X}}}(\cdot,\cdot)$ over $P_{\mathcal{X}}$, the empirical target risk $R_{\mathbb{T}}(\cdot)$ over $P$,  the empirical target disagreement $R_{\mathbb{T}_{\mathcal{X}}}(\cdot,\cdot)$ over $P_{\mathcal{X}}$ are defined in a similar way.    

\subsection{Simple Pac-Bayesian Bounds}\label{Simple Pac-Bayesian bounds}

The PAC-Bayesian approach abbreviated Pac-Bayes is analysis techniques of generalization in the theory of statistical learning. PAC-Bayes inequalities were introduced by \cite{SHAW97}, and \cite{MCA98a,MCA99b}; and further formalised
\cite{CAT07,CAT04,CAT03} and other (see \cite{GUE19} for a recent survey and \cite{ALQ21} for an introduction to the
field). It provides PAC (probably approximately correct, Valiant, 1984) generalization bounds by expressing a trade-off between the empirical risk on the training set
and a measure of complexity of the predictors class as a weighted majority vote over a set of functions from the hypothesis space $\mathcal{H}$. 

In this section, we recall the general PAC–Bayesian generalization bounds in the setting of binary classification with the $\monzeroun$-loss defined in the above section. To
derive such a generalization bound, one assumes a prior distribution $\mathcal{P}$ over $\mathcal{H}$, which models an a priori belief on the hypothesis from $\mathcal{H}$ before the observation of the
training sample $\mathbb{S} \sim Q^{m}$. Given the training sample $\mathbb{S}$, the learner aims at finding a posterior distribution $\mathcal{Q}$ over $\mathcal{H}$ that leads to a well-performing $\mathcal{Q}$-weighted majority vote $\mathcal{B}_{\mathcal{Q}}$ (also called the bayes classifier) defined as: 
\begin{equation}
    \mathcal{B}_{\mathcal{Q}}(\bm{x}) = \; \text{sign}\Big[\underset{h\sim \mathcal{Q}}{\mathbb{E}}h(x)\Big].
\end{equation}
We want to learn $\mathcal{Q}$ over $\mathcal{H}$ such that it minimizes the true risk $\mathcal{R}_{Q}(\mathcal{B}_{\mathcal{Q}})$ of the $\mathcal{Q}$-weighted majority vote. However, the risk of $B_{\mathcal{Q}}$ is known to be NP-hard, therefore PAC–Bayesian generalization bounds do not directly focus on the risk of $\mathcal{B}_{\mathcal{Q}}$. Instead, it gives an upper bound over the expectation over $\mathcal{Q}$ of all the individual hypothesis true risk called the expected/true Gibbs risk: 
\begin{equation}
    \mathcal{R}_{Q}(G_{\mathcal{Q}}) = \underset{h\sim \mathcal{Q}}{\mathbb{E}}\Big[\mathcal{R}_{Q}(h)\Big].
\end{equation}

The expected Gibbs risk is closely related to the deterministic $\mathcal{Q}$-weighted majority vote. Indeed, if $\mathcal{B}_{\mathcal{Q}}(\cdot)$ misclassifies $\bm{x} \in \mathcal{X}$, then at least half of the classifiers (under measure $\mathcal{Q}$) make a prediction error on $\bm{x}$. Therefore, we have $ \mathcal{R}_{Q}(\mathcal{B}_{\mathcal{Q}}) \leq2 \,\mathcal{R}_{Q}(G_{\mathcal{Q}})$. Another result on the relation between $ \mathcal{R}_{Q}(\mathcal{B}_{\mathcal{Q}})$ and $ \mathcal{R}_{Q}(G_{\mathcal{Q}})$ know as C-bound (\cite{LAC06}) and defined as:

\begin{equation} \label{C-bound}
    \mathcal{R}_{Q}(\mathcal{B}_{\mathcal{Q}}) \leq 1 - \frac{\big(1-2\,\mathcal{R}_{Q}(G_{\mathcal{Q}})\big)^{2}}{1-2\,d_{Q_{\mathcal{X}}}(\mathcal{Q})},
\end{equation}
where $d_{Q_{\mathcal{X}}}(\mathcal{Q})$\label{expected_disagreement} corresponds to the expected disagreement between pairs of voters on the marginal distribution $Q_\mathcal{X}$:
\begin{equation}
    d_{Q_{\mathcal{X}}}(\mathcal{Q}) = \underset{(h,h')\sim\mathcal{Q}^{2}}{\mathbb{E}}\Big[\mathcal{R}_{Q_{\mathcal{X}}}\big(h,h'\big)\Big].
\end{equation}

The authors in \cite{LAC06} observed that in a binary classification context, the expected disagreement $d_{Q_{\mathcal{X}}}(\mathcal{Q})$ is closely related
to the notion of expected joint error $e_{Q}(\mathcal{Q})$ between pairs of voters, $e_{Q}(\mathcal{Q})$:
\begin{equation}
    e_{Q}(\mathcal{Q}) = \underset{(h,h')\sim\mathcal{Q}^{2}}{\mathbb{E}}\;\underset{(\bm{x},y)\sim Q}{\mathbb{E}} \Big[\mathcal{L}_{\monzeroun}\big(h(\bm{x}),y\big)\times \mathcal{L}_{\monzeroun}\big(h'(\bm{x}),y\big)\Big].
\end{equation}
Indeed, for all distribution $Q$ on $\mathcal{X} \times \mathcal{Y}$ and a distribution $\mathcal{Q}$ on $\mathcal{H}$, we can decompose $
\mathcal{R}_{Q}(G_{\mathcal{Q}})$ as follows:

\begin{equation}
    \mathcal{R}_{P}(G_{\mathcal{Q}}) = \frac{1}{2}\,d_{P_{\mathcal{X}}}(\mathcal{Q}) + e_{P}(\mathcal{Q}).
\end{equation}

The PAC-Bayesian theory, suggests that minimizing the expected Gibbs risk $\mathcal{R}_{Q}(G_{\mathcal{Q}})$ can be done by minimizing the trade-off between the empirical Gibbs risk $\mathcal{R}_\mathbb{S}(G_{\mathcal{Q}})$ and Kullback–Leibler divergence minimization $D_{KL}(\mathcal{Q}||\mathcal{P})$. Note that PAC–Bayesian generalization bounds do not directly take into account the complexity of the hypothesis class $\mathcal{H}$, but measure the deviation between the prior
distribution $\mathcal{P}$ and the posterior distribution $\mathcal{Q}$ on $\mathcal{H}$ through the Kullback–Leibler
divergence. In the following, we present three main PAC-Bayesian bound proposed by \cite{MCA99b}; \cite{SEE02,LAN05}; \cite{CAT07}.

Firstly, \cite{SEE02}; and \cite{LAN05} propose the following PAC-Bayesian theorem in which the tradeoff between the complexity and the risk is handled by $D_{KL}(\cdot||\cdot)$:

\begin{theorem}{(Seeger \cite{SEE02}; Langford \cite{LAN05})}\label{pac_bayes_seeger}
For any distribution $Q$ over $\mathcal{X} \times \mathcal{Y}$, any set of hypotheses $\mathcal{H}$, and any prior distribution $\mathcal{P}$ over $\mathcal{H}$, any $\delta \in (0, 1]$, with a probability at least $1-\delta$ over the choice of $\mathbb{S} \sim Q^{m}$, for every $\mathcal{Q}$ over $\mathcal{H}$, we have:
    \begin{equation}  D_{\textnormal{KL}}\Big(\mathcal{R}_{\mathbb{S}}(G_{\mathcal{Q}})\Big|\Big|\mathcal{R}_{Q}(G_{\mathcal{Q}})\Big) \leq \frac{1}{m}\Bigg[D_{KL}(\mathcal{Q}||\mathcal{P}) + \ln{\frac{2\sqrt{m}}{\delta}}\Bigg].
    \end{equation}
\end{theorem}

As stated in \cite{GER15}, the PAC-Bayesian theorem in this form provides a precise bound, particularly for low empirical Gibbs risk. Nevertheless, the inclusion of the $D_{KL}\Big(\mathcal{R}_{\mathbb{S}}(G_{\mathcal{Q}})\Big|\Big|\mathcal{R}_{Q}(G_{\mathcal{Q}})\Big)$ term makes the interpretation of this bound challenging since it does not provide a direct relationship between the empirical Gibbs risk $\mathcal{R}_{\mathbb{S}}(G_{\mathcal{Q}})$ and the true Gibbs risk $\mathcal{R}_{Q}(G_{\mathcal{Q}})$ in a closed form.

Thus, finding the distribution $\mathcal{Q}$ that minimizes the bound on $\mathcal{R}_{Q}(G_{\mathcal{Q}})$ given by Theorem \ref{pac_bayes_seeger} could be a difficult task from an algorithmic standpoint. However, the first proposed version of the PAC-Bayes theorem (McAllester \cite{MCA99b}) offers a simpler interpretation by linking $\mathcal{R}_{\mathbb{S}}(G_{\mathcal{Q}})$ and $\mathcal{R}_{Q}(G_{\mathcal{Q}})$ through a linear relation. The following PAC-Bayesian theorem is derived from McAllester \cite{MCA99b}:

\begin{theorem}{(McAllester \cite{MCA99b})}\label{pac_bayes_mcallester}
For any distribution $Q$ over $\mathcal{X} \times \mathcal{Y}$, any set of hypotheses $\mathcal{H}$, and any prior distribution $\mathcal{P}$ over $\mathcal{H}$, any $\delta \in (0, 1]$, with a probability at least $1-\delta$ over the choice of $\mathbb{S} \sim Q^{m}$, for every $\mathcal{Q}$ over $\mathcal{H}$, we have:
\begin{equation}
\Big|\mathcal{R}_{Q}(G_{\mathcal{Q}}) - \mathcal{R}_{\mathbb{S}}(G_{\mathcal{Q}})\Big| \leq \sqrt{\frac{1}{2m}\Bigg[D_{KL}(\mathcal{Q}||\mathcal{P}) + \ln{\frac{2\sqrt{m}}{\delta}}\Bigg]}.
\end{equation}
\end{theorem}

As stated in \cite{GER15}, theorems \ref{pac_bayes_seeger} and \ref{pac_bayes_mcallester} propose that to reduce the expected Gibbs risk, a learning algorithm needs to balance between minimizing the empirical Gibbs risk $\mathcal{R}_{\mathbb{S}}(G_{\mathbf{Q}})$ and minimizing the KL-divergence $D_{KL}(\cdot \Vert \cdot)$.

The explicit control of this trade-off is achievable through theorem \ref{pac_bayes_catoni}, which is a PAC-Bayesian result introduced by \cite{CAT07} and defined with a hyperparameter (referred to as c in theorem \ref{pac_bayes_catoni}). This result seems to be a convenient approach for developing PAC-Bayesian algorithms. We are presenting a simplified version of this result as suggested by \cite{GER09}.

\begin{theorem}{(Catoni \cite{CAT07})}\label{pac_bayes_catoni}
    For any distribution $Q$ on $\mathcal{X} \times \mathcal{Y}$, for any set of voters $\mathcal{H}$, \,for any prior distribution $\mathcal{P}$ on $\mathcal{H}$, \,for any $\delta \in (0, 1]$, and any real number $c > 0$, with probability at least $1 - \delta$ over the random choice of the sample $\mathbb{S} \sim Q^{m}$, for every posterior distribution $\mathcal{Q}$ on $\mathcal{H}$, we have
    \begin{equation}
        \mathcal{R}_{Q}(G_{\mathcal{Q}}) \leq \frac{c}{1-e^{-c}}\Bigg[\mathcal{R}_{\mathbb{S}}(G_{\mathcal{Q}}) + \frac{D_{KL}(\mathcal{Q}||\mathcal{P})+\ln \frac{1}{\delta})}{m\times\omega}\Bigg].
    \end{equation}
\end{theorem}

In this section we reviewed the principal Pac-Bayes bounds analysis with no adaptation and no concept of multi-view learning. In the next sections we will present the principal bounds of generalisation for those different concepts.

\subsection{Analysis of Domain Adaptation Pac-Bayesian Bounds}  In this section, we recall the work done by \cite{GER13,GER15} on how the PAC–Bayesian theory can help to theoretically understand domain adaptation through the weighted majority vote learning point of view.

The first Pac-Bayesian generalization bound for domain adaptation was introduced in \cite{GER13}. The authors defined a divergence measure that follows the idea of C-bound equation \ref{C-bound}. Thus, Germain et al. underlined that the domains $\mathbb{D}_{\mathcal{S}}$ and $\mathbb{D}_{\mathcal{T}}$ are close according to $\mathcal{Q}$ if
the expected disagreement over the two domains tends to be close. More formally, if $\mathcal{R}_{Q}(G_{\mathcal{Q}})$ and $\mathcal{R}_{P}(G_{\mathcal{Q}})$ are similar, then and $\mathcal{R}_{Q}(\mathcal{B}_{\mathcal{Q}})$ and $\mathcal{R}_{P}(\mathcal{B}_{\mathcal{Q}})$ are similar when $d_{Q_{\mathcal{X}}}(\mathcal{Q})$
and $d_{P_{\mathcal{X}}}(\mathcal{Q})$ are also similar. In this way, the authors introduced the following domain disagreement pseudometric \footnote{A pseudometric $d$ is a metric for wich the property $d(x,y)=0 \Leftrightarrow x=y$ is relaxed to $d(x,y) = 0  \Leftarrow x=y$}.

\begin{definition}\label{disagreement definition}
    Let $\mathcal{H}$ be a hypothesis class. For any marginal distributions $Q_{\mathcal{X}}$ and $P_{\mathcal{X}}$ over $\mathcal{X}$, any distribution $\mathcal{Q}$ on $\mathcal{H}$, the domain disagreement $dis_{\mathcal{Q}}(Q_{\mathcal{X}}, P_{\mathcal{X}})$ between $Q_{\mathcal{X}}$ and $P_{\mathcal{X}}$ is defined by:
    \begin{equation}
        dis_{\mathcal{Q}}(Q_{\mathcal{X}},P_{\mathcal{X}}) = \Big|d_{P_{\mathcal{X}}}(\mathcal{Q}) - d_{Q_{\mathcal{X}}}(\mathcal{Q})\Big|.
    \end{equation}
\end{definition}

Note that $dis_{\mathcal{Q}}(\cdot,\cdot)$ is symmetric and fulfills the triangle inequality \cite{GER13}. Note that for the sake of simplicity, we suppose that m = n, i.e., the size of $\mathbb{S}/\mathbb{S}_{\mathcal{X}}$ and $\mathbb{T}/\mathbb{T}_{\mathcal{X}}$ are equal. The authors showed that $dis_{\mathcal{Q}}(Q_{\mathcal{X}}, P_{\mathcal{X}})$ can be bounded in terms of the classical PAC-Bayesian quantities and propose the following theorem (the theorem and it's proof can be found in \cite{GER13,GER15}):
\begin{theorem}{(Germain \cite{GER13,GER15})}\label{theorem disq}
    For any distributions $Q_{\mathcal{X}}$ and $P_{\mathcal{X}}$ over $\mathcal{X}$, any set of hypothesis $\mathcal{H}$, any prior distribution $\mathcal{P}$ over $\mathcal{H}$, any $\delta \in (0, 1]$, and any real number $ \alpha > 0$, with a probability at least $1-\delta$ over the choice of $(\mathbb{S}_{\mathcal{X}} \times \mathbb{T}_{\mathcal{X}}) \sim (Q_{\mathcal{X}} \times P_{\mathcal{X}} )^{m=n}$, for every $\mathcal{Q}$ on $\mathcal{H}$, we have,
    \begin{displaymath}
        dis_{\mathcal{Q}}(Q_{\mathcal{X}},P_{\mathcal{X}}) \leq \frac{2\alpha}{1-e^{-2\alpha}} \Bigg[dis_{\mathcal{Q}}(\mathbb{S}_{\mathcal{X}},\mathbb{T}_{\mathcal{X}}) + \frac{2D_{\textnormal{KL}}(\mathcal{Q}\Vert \mathcal{P})+\ln{\frac{2}{\delta}}}{m\times \alpha} + 1 \Bigg]-1.
    \end{displaymath}

\end{theorem}

 Note that the authors in \cite{GER13,GER15} propose the theorem \ref{theorem disq} as "Catoni's type" with $c=2\alpha$. Indeed, as described in section \ref{Simple Pac-Bayesian bounds} and in \cite{GER13,GER15} "Catoni’s type” with  PAC-Bayesian bound present interesting characteristics. First, its minimization is closely related to the minimization problem associated with the SVM when $\mathcal{Q}$ is an isotropic Gaussian over the space of linear classifiers (Germain et al. \cite{GER09}). Second, the value $c=2\alpha$ allows to control the trade-off between the empirical risk and the complexity term $D_{\textnormal{KL}}(\cdot \Vert \cdot)$. 
 
 Thereby, from this domain’s divergence, the authors proved the following domain
adaptation bound (the theorem and its proof can be found in \cite{GER15}).
\begin{theorem}{(Germain et al. \cite{GER13,GER15}).}\label{theorem_bound_germ13}
Let $\mathcal{H}$ be a hypothesis class. We have, $\forall\; \mathcal{Q}$ on $\mathcal{H}$,

\begin{equation}
        \mathcal{R}_{P}(G_{\mathcal{Q}}) \leq \mathcal{R}_{Q}(G_{\mathcal{Q}}) + \frac{1}{2} dis_{\mathcal{Q}}(Q_{\mathcal{X}},P_{\mathcal{X}}) + \lambda_{\mathcal{Q}},
\end{equation}
    
\end{theorem}
where $\lambda_{\mathcal{Q}}$ is the deviation between the expected joint errors between pairs for voters ad pairs of views defined in section \ref{Simple Pac-Bayesian bounds}
on the target and source domains, which is defined as $\lambda_{\mathcal{Q}} = \Big|e_{P}(\mathcal{Q}) - e_{Q}(\mathcal{Q})\Big|$ and where \\ 
$e_{P}(\mathcal{Q}) = \mathbb{E}_{(\bm{x},y)\sim P}\;\mathbb{E}_{(h,h')\sim\mathcal{Q}^{2}_{v}}\;\big[\mathcal{L}_{\monzeroun}\big(h(x),y\big)\times\mathcal{L}_{\monzeroun}\big(h'(x),y\big)\big]$, \\$e_{Q}(\mathcal{Q}) = \mathbb{E}_{(\bm{x},y)\sim Q}\;\mathbb{E}_{(h,h')\sim\mathcal{Q}^{2}_{v}}\;\big[\mathcal{L}_{\monzeroun}\big(h(x),y\big)\times\mathcal{L}_{\monzeroun}\big(h'(x),y\big)\big]$.\\

The bound can be seen as a trade-off between different quantities. $R_{Q}(G_{\mathcal{Q}})$ and $dis_{\mathcal{Q}}(Q_{\mathcal{X}},P_{\mathcal{X}})$ are comparable to the first two terms of the DA-bound in \cite{BEN10a}: $\mathcal{R}_{Q}(G_{\mathcal{Q}})$ is the $\mathcal{Q}$-average risk over $\mathcal{H}$ on the source domain, and $dis_{\mathcal{Q}}(Q_{\mathcal{X}},P_{\mathcal{X}})$ measures the $\mathcal{Q}$-average disagreement between the marginals.

In this section, we presented the principal bounds for Pac-Bayesian domain adaptation, in the next section we will discuss about Pac-Bayesian bounds in the multi-view learning setting.

\subsection{Analysis of Pac-Bayesian Multi-view Bounds}\label{section_mv_bound}
 
First, the authors in \cite{SUN17} provided PAC-Bayesian bounds over the concatenation of the views, using priors that reflect how well the views agree on average over all examples, and deduced a SVM-like learning algorithm from this framework. However, this concatenation is designed for two views and kernel method, it is not generalizable to other methods. A more general framework of Pac-Bayesian bounds for multi-views was introduced in 
\cite{GOYAL17}. In the paper, the authors inroduced the two-level multiview approach. For each view $v \in [\![\mathcal{V}]\!]$, they consider a view-specific set $\mathcal{H}_{v}$ of voters $h : \mathcal{X}^{v} \rightarrow Y$, and a prior distribution $\mathcal{P}_{v}$ on $\mathcal{H}_{v}$. Given a hyper-prior distribution $\pi$ over the views $[\![\mathcal{V}]\!]$. In the paper, PAC-Bayesian learner objective has two parts. The first part is finding a posterior distribution $\mathcal{Q}_{v}$ over $\mathcal{H}_{v}, \forall \; v \in [\![\mathcal{V}]\!]$; The second is finding a hyper-posterior $\rho$ distribution on the set of views $[\![\mathcal{V}]\!]$. Thereby, \cite{GOYAL17} defined the multi-view weighted majority vote $\mathcal{B}_{\rho}^{\textnormal{MV}}$ as: 

\begin{equation}
    \mathcal{B}_{\rho}^{\textnormal{MV}}(\bm{x}) = \; \text{sign}\Big[\underset{v\sim \rho}{\mathbb{E}}\,\underset{h\sim \mathcal{Q}_{v}}{\mathbb{E}}h(\bm{x}^{v})\Big].
\end{equation}

Thus, the authors propose to build a learner that intends to construct posterior and hyperposterior distributions that minimize the actual risk $\mathcal{R}_{D}(\mathcal{B} ^{\textnormal{MV}}_{\rho})$ of the multiview weighted majority vote defined as: 

\begin{equation}
    \mathcal{R}_{Q}(\mathcal{B} ^{\textnormal{MV}}_{\rho}) = \underset{(\bm{x}^{v},y)\sim Q}{\mathbb{E}} \Big[\mathcal{L}_{\monzeroun}\big(\mathcal{B}^{\textnormal{MV}}_{\rho}(\bm{x}^{v}),y\big)\Big].
\end{equation}

As stated in Section \ref{Simple Pac-Bayesian bounds} the Gibbs risk is related to the expected disagreement $d_{Q_{\mathcal{X}}}(\mathcal{Q})$ and expected joint error $e_{Q_{\mathcal{X}}}(\mathcal{Q})$. Goyal et al. \cite{GOYAL17} note that the stochastic Gibbs classifier $G^{\textnormal{MV}}_{\rho}$ defined as follows in the multiview setting, follows the same idea and it can be rewritten in terms of expected disagreement $d^{\textnormal{MV}}_{Q}(\rho)$ and expected joint error $e^{\textnormal{MV}}_{Q}(\rho)$:

\begin{equation}
\begin{split}
    \mathcal{R}_{Q}(G^{\textnormal{MV}}_{\rho}) &= \frac{1}{2}d^{\textnormal{MV}}_{Q}(\rho) + e^{\textnormal{MV}}_{Q}(\rho), \\
    \text{where}\; d^{MV}_{Q_{\mathcal{X}}}(\rho) &= \mathbb{E}_{x^{v}\sim Q_{\mathcal{X}}}\;\mathbb{E}_{(v,v')\sim\rho^{2}}\;\mathbb{E}_{(h,h')\sim\mathcal{Q}^{2}_{v}}\big[\mathcal{L}_{\monzeroun}\big(h(x^{v}),h'(x^{v})\big)\big], \\
    \text{and}\; e^{MV}_{Q}(\rho) &= \mathbb{E}_{x^{v}\sim Q}\;\mathbb{E}_{(v,v^{'})\sim \rho^{2}}\;\mathbb{E}_{(h,h')\sim\mathcal{Q}^{2}_{v}}\;\big[\mathcal{L}_{\monzeroun}\big(h(x^{v}),y\big)\times \mathcal{L}_{\monzeroun}\big(h'(x^{v}),y\big)\big]. 
\end{split}
\end{equation}

As in the single-view PAC-Bayesian setting, \cite{GOYAL17} note that the multiview weighted majority vote $\mathcal{B}^{\textnormal{MV}}_{\rho}$ is closely related to the stochastic multiview Gibbs classifier $G^{\textnormal{MV}}_{\rho}$, and a generalization bound for $G^{\textnormal{MV}}_{\rho}$ gives rise to a generalization bound for $\mathcal{B}^{\textnormal{MV}}_{\rho}$. Moreover, the authors extend the C-Bound \ref{C-bound} to multiview setting by Lemma \ref{C_bounds_MV} below (the theorem
and it’s proof can be found in \cite{GOYAL17}). 

\begin{lemma}\label{C_bounds_MV}{(Goyal et al. \cite{GOYAL17}).}
Let $V \geq 2$ be the number of views. For all posterior $\{\mathcal{Q}_{v}\}^{V}_{v=1}$ and hyper-posterior $\rho$ distribution, if $ \mathcal{R}_{Q}(G^{\textnormal{MV}}_{\rho})< \frac{1}{2}$, then we have:
\begin{equation}
\begin{split}
    \mathcal{R}_{Q}(\mathcal{B}^{\textnormal{MV}}_{\rho}) \leq 1 - \frac{(1-2\,\mathcal{R}_{Q}\big(G^{\textnormal{MV}}_{\rho})\big)^{2}}{1-2d^{MV}_{Q_{\mathcal{X}}}(\rho)} \leq 1 - \frac{(1-2\,\mathbb{E}_{v\sim \rho}\mathcal{R}_{Q}\big(G_{\mathcal{Q}_{v}})\big)^{2}}{1-2\,\mathbb{E}_{v\sim \rho}\,d_{Q_{\mathcal{X}}}(\mathcal{Q}_{v})}.
\end{split}
\end{equation}
\end{lemma}

The authors in the paper provide general multi-view PAC-Bayesian theorems and derive also a generalization bound with the approaches of \cite{MCA98a}; \cite{SEE02}; \cite{LAN05}; and \cite{CAT07} introduce in section \ref{Simple Pac-Bayesian bounds}. \ref{pac_bayes_catoni} (the theorems and its proof can be found in \cite{GOYAL17}). The main difference between Goyal et al.'s bounds to theorems  \cite{MCA98a}; \cite{SEE02}; \cite{LAN05}; and \cite{CAT07} relies on the introduction of view-specific prior and posterior distributions, which mainly leads to an additional term $\mathbb{E}_{v\sim \rho}\;D_{\textnormal{KL}}(\mathcal{Q}_{v}\Vert \mathcal{P}_{v})$, expressed as the expectation of the view-specific Kullback-Leibler divergence term over the views $[\![\mathcal{V}]\!]$ according to the hyper-posterior distribution $\rho$.

\section{Analysis of Unsupervised Multi-view  Domain Adaptation PAC-Bayesian Bounds}
As mentioned in the introduction section, multi-view learning have paid little attention in domain adaptation studies. Besides, the studies propose a empirical analysis. The first theoretical analysis was proposed in \cite{YANG14} where the authors introduce the notion of disagreement between views. The work propose to minimize the disagreement on unlabeled data in the target domain to enhance the consistency in the adaptation method. Unfortunately, the work does not benefit from any generalization guarantees given by an upper bound on the risk with respect to the target distribution. Furthermore, the disagreement defined in the paper doesn't take account the weighted majority for the hypotheses and views. In this way, disagreement may cause a negative transfer if noisy views or hypotheses are not weighted.

In this section we propose to introduce the concept of multi-view learning in the DA with generalization Pac-Bayesian guarantees. Then, we adapt  the divergence proposed by Germain et al., \cite{GER13,GER16} with the concept of multi-view weighted majority vote introduced in \cite{GOYAL17}.  In a second phase, we propose a general multi-view domain adaptation PAC-Bayesian theorem to upper-bound our divergence. Then, we present specialization of our theorem to the classical approaches as presented in section \ref{Simple Pac-Bayesian bounds}. Finally, we propose a Pac-Bayesian domain adaptation bound in the multi-view setting.

\subsection{Multi-view Domain Disagreement} Germain et al. \cite{GER13} and Mansour et al.\cite{MAN09a} propose a divergence measure that is based on the expected disagreement over the two domains. In the idea of measure disagreement we propose to adapt the definition \ref{disagreement definition} proposed by \cite{GER13} to multi-view learning. Thus, we define the multi-view domain disagreement as follows:

\begin{definition}{(Multi-view domain disagreement)}\label{Multi-view domain disagreement}
$\forall \, v \in [\![\mathcal{V}]\!]$, for any set of voters $\mathcal{H}_{v}$ for any marginal distributions $Q_{\mathcal{X}}$ and $P_{\mathcal{X}}$ over $\mathcal{X}$, any set of posterior distribution $\{\mathcal{Q}_{v}\}_{v=1}^{V}$ on $\mathcal{H}_{v}$, for any hyper-posterior distribution $\rho$ over $[\![\mathcal{V}]\!]$,   
    the multi-view domain disagreement $dis_{\rho}^{\textnormal{MV}}(Q_{\mathcal{X}},P_{\mathcal{X}})$ between $Q_{\mathcal{X}}$ and $P_{\mathcal{X}}$ is defined by:
    \begin{equation}
        dis_{\rho}^{\textnormal{MV}}(Q_{\mathcal{X}},P_{\mathcal{X}}) = \Big|d_{P_{\mathcal{X}}}^{\textnormal{MV}}(\rho) - d_{Q_{\mathcal{X}}}^{\textnormal{MV}}(\rho)\Big|,
    \end{equation}

\end{definition}
    where $ d^{\textnormal{MV}}_{Q_{\mathcal{X}}}(\rho)$, 
    $ d^{\textnormal{MV}}_{P_{\mathcal{X}}}(\rho)$, are expected disagreement defined in \ref{expected_disagreement} and defined as follows in our multiview setting \cite{GOYAL17}.\\

From this domain’s divergence, we can find the same domain
adaptation bound \ref{theorem_bound_germ13} presented  in section \ref{General PAC-Bayes bounds} with an adaptation to multi-view learning: 

\begin{theorem}\label{MVDA_bound_hennequin}
$\forall \, v \in [\![\mathcal{V}]\!]$, for any distributions $Q$ and $P$ over $\mathcal{X}\times \mathcal{Y}$, for any set of voters $\mathcal{H}_{v}$, for any marginal distributions $Q_{\mathcal{X}}$ and $P_{\mathcal{X}}$ over $\mathcal{X}$, any set of posterior distribution $\{\mathcal{Q}_{v}\}_{v=1}^{V}$ on $\mathcal{H}_{v}$, for any hyper-posterior distribution $\rho$ over $[\![\mathcal{V}]\!]$, we have:
    \begin{equation}
        \mathcal{R}_{P}(G_{\rho}^{\textnormal{MV}}) \leq \mathcal{R}_{Q}(G_{\rho}^{\textnormal{MV}}) + \frac{1}{2} dis_{\rho}^{\textnormal{MV}}(Q_{\mathcal{X}},P_{\mathcal{X}}) + \lambda_{\rho},
    \end{equation}
\end{theorem}

where $\lambda_{\rho}$ is the deviation between the expected joint errors between pairs for voters ad pairs of views defined in section \ref{section_mv_bound}
on the target and source domains, which is defined as $\lambda_{\rho} = \Big|e_{P}^{\text{MV}}(\rho) - e_{Q}^{\text{MV}}(\rho)\Big|$ and where \\ 
$e_{P}^{\text{MV}}(\rho) = \mathbb{E}_{(\bm{x},y)\sim P}\;\mathbb{E}_{(v,v^{'})\sim \rho^{2}}\;\mathbb{E}_{(h,h')\sim\mathcal{Q}^{2}_{v}}\;\big[\mathcal{L}_{\monzeroun}\big(h(x),y\big)\times\mathcal{L}_{\monzeroun}\big(h'(x),y\big)\big]$, \\$e_{Q}^{\text{MV}}(\rho) = \mathbb{E}_{(\bm{x},y)\sim Q}\;\mathbb{E}_{(v,v^{'})\sim \rho^{2}}\;\mathbb{E}_{(h,h')\sim\mathcal{Q}^{2}_{v}}\;\big[\mathcal{L}_{\monzeroun}\big(h(x),y\big)\times\mathcal{L}_{\monzeroun}\big(h'(x),y\big)\big]$. 
\begin{proof}
    The proof borrows the straightforward proof technique of Theorem 9 in \cite{GER15}. A detailed proof is available in additional materials. 
\end{proof}

Note that in the context of adaptation, a key question arises: How does the expected target risk, $\mathcal{R}_{P}(h)$, differs from the expected source risk $\mathcal{R}_{Q}(h)$ ? In the context of the Pac-Bayesian multi-view domain adaptation we want to know how the expected Gibbs risk on the target distribution, $\mathcal{R}_{P}(G_{\rho}^{\textnormal{MV}})$, differs from the expected Gibbs risk on the source distribution, $\mathcal{R}_{Q}(G_{\rho}^{\textnormal{MV}})$. The above theorem addresses this question, it proves that the error achieved by hypothesis in the source domain upper bounds the Gibbs risk on the target domain plus the distance between their distributions, here Multi-view domain disagreement \ref{Multi-view domain disagreement}, and the deviation between the expected joint errors between pairs for voters on the target and source domains. Note that the theorem \ref{MVDA_bound_hennequin} is very similar to the theorem \ref{theorem_bound_germ13}. In fact, we trivially have $dis_{\rho}^{\textnormal{MV}}(Q_{\mathcal{X}},P_{\mathcal{X}}) \leq \mathbb{E}_{(v,v')\sim \rho^{2}}\; dis_{\mathcal{Q}}(Q_{\mathcal{X}},P_{\mathcal{X}})$, that inequality exhibits the role of diversity among the views thanks to the disagreement’s expectation over the views as showed in the equation \ref{C_bounds_MV}.

\subsection{General Multi-view Domain Disagreement Pac-Bayesian Theorem}

In this section we show that $dis_{\rho}^{\textnormal{MV}}(Q_{\mathcal{X}},P_{\mathcal{X}})$ can be bounded in terms of the general PAC-Bayesian quantities. Before to present the bound we present a general Pac-bayesian theorem introduced by \cite{GER09,GER15a} and further generalized to 
the Rényi divergence by \cite{BEG16}; The authors in \cite{GOYAL17} also propose a variation of the general PAC-Bayesian theorem of \cite{GER09,GER15a}; It takes the form of an upper bound on the deviation between the true and empirical risk of the Gibbs classifier, according to a convex function $\Delta :[0, 1]^{2}\rightarrow \mathbb{R}$. 

\begin{theorem}{(General PAC-Bayes bounds)}\label{General PAC-Bayes bounds}
    For any distribution $Q$ on $\mathcal{X} \times \mathcal{Y}$, for any set $\mathcal{H}$ of voters $\mathcal{X} \rightarrow \{-1, 1\}$, for any prior distribution $\mathcal{P}$ on $\mathcal{H}$, for any $\delta \in (0, 1]$, for any $m > 0$, and for any convex function $\Delta : [0, 1]^{2} \rightarrow \mathbb{R}$, with probability at least $1-\delta$ over the choice of $\mathbb{S} \sim Q^{m}$, we have:\\
(Germain et al. \cite{GER09})
    \begin{equation}
    \begin{split}
    &\underset{\mathbb{S}\sim Q^{m}}{\mathbb{P}}\Bigg(\forall \mathcal{Q} \text{ on } \mathcal{H},\; \Delta \Big(\mathcal{R}_{\mathbb{S}}(G_{\mathcal{Q}}),\mathcal{R}_{Q}(G_{\mathcal{Q}})\Big) \\ &\leq \frac{1}{m'}\Bigg[D_{\textnormal{KL}}(\mathcal{Q}\Vert\mathcal{P})\ln{\bigg(\frac{1}{\delta}\, \underset{\mathbb{S} \sim Q^{m}}{\mathbb{E}}\underset{h\sim \mathcal{P}}{\mathbb{E}} e^{m\,\Delta(\mathcal{R}_{\mathbb{S}}(h),\mathcal{R}_{D}(h))}\bigg)} \Bigg]\Bigg) \geq 1-\delta,
    \end{split}
    \end{equation}
    (Bégin et al.\cite{BEG16})
        \begin{equation}
        \begin{split}
    &\underset{\mathbb{S}\sim Q^{m}}{\mathbb{P}}\Bigg(\forall \mathcal{Q} \text{ on } \mathcal{H},\;
    \ln{\Big[\Delta \Big(\mathcal{R}_{\mathbb{S}}(G_{\mathcal{Q}}),\mathcal{R}_{Q}(G_{\mathcal{Q}})\Big)}\Big] \\
    &\leq \frac{1}{\alpha'}\Bigg[D_{\alpha}(\mathcal{Q}\Vert\mathcal{P})\ln{\bigg(\frac{1}{\delta}\, \underset{\mathbb{S} \sim Q^{m}}{\mathbb{E}}\underset{h\sim \mathcal{P}}{\mathbb{E}} e^{m\,\Delta(\mathcal{R}_{\mathbb{S}}(h),\mathcal{R}_{D}(h))}\bigg)} \Bigg]\Bigg) \geq 1-\delta,
    \end{split}
    \end{equation}
(Goyal et al. \cite{GOYAL17})
\begin{equation}\label{General PAC-Bayes bound goyal}
\begin{split}
    &\Delta \Big(\,\underset{\mathbb{S}\sim Q^{m}}{\mathbb{E}}\mathcal{R}_{\mathbb{S}}(G_{\mathcal{Q}_{\mathbb{S}}}), \underset{\mathbb{S}\sim Q^{m}}{\mathbb{E}} \mathcal{R}_{Q}(G_{\mathcal{Q}_{\mathbb{S}}})\Big)\\ &\leq \frac{1}{m}\Bigg[\underset{\mathbb{S}\sim Q^{m}}{\mathbb{E}}D_{\textnormal{KL}}(\mathcal{Q}_{\mathbb{S}}\Vert\mathcal{P}) + \ln{\bigg(\frac{1}{\delta}\, \underset{\mathbb{S} \sim Q^{m}}{\mathbb{E}}\underset{h\sim \mathcal{P}}{\mathbb{E}} e^{m\,\Delta(\mathcal{R}_{\mathbb{S}}(h),\mathcal{R}_{D}(h))}\bigg)} \Bigg],
\end{split}
    \end{equation}

\end{theorem}
    where $\text{D}_{\alpha}(\mathcal{Q}\Vert\mathcal{P}) = \frac{1}{\alpha-1}\ln{\bigg[\underset{h \sim \mathcal{P}}{\mathbb{E}}\bigg[\frac{\mathcal{Q}(h)}{\mathcal{P}(h)}\bigg]^{\alpha}\bigg]}$, the Rényi divergence $(\alpha >1)$ and $\alpha' = \frac{\alpha}{\alpha - 1}$.
We note that the Pac-Bayesian bound \ref{General PAC-Bayes bound goyal} from \cite{GOYAL17} bounds $\mathbb{E}_{\mathbb{S}\sim Q^{m}}\mathcal{\mathcal{R}}_{Q}(G_{\mathcal{Q}_{\mathbb{S}}})$, where $\mathcal{Q}_{\mathbb{S}}$ is the posterior distribution outputted by a given learning algorithm after observing the learning sample $\mathbb{S}$. Whereas PAC-Bayesian bounds from \cite{GER09,GER15a} and \cite{BEG16} bounds $\mathcal{R}_{Q}(G_{\mathcal{Q}})$ uniformly for all distribution $Q$, with high probability over the draw of $\mathbb{S} \sim Q^{m}$. Thereby, Goyal et al. \cite{GOYAL17}'s theorem  has the advantage to involve an expectation over all the possible learning samples (of a given size) in bounds itself. Moreover, the authors in \cite{GOYAL17} adapt their general Pac-Bayesian bound \ref{General PAC-Bayes bound goyal} to multi-view setting. 
Then, we propose the following theorem as a generalization of Theorem \ref{General PAC-Bayes bounds}; It takes the form of an upper bound on the deviation between the true risk of the Multi-view domain disagreement \ref{Multi-view domain disagreement} $\mathbb{E}_{\mathbb{S}\sim Q}\;dis_{\rho_{\mathbb{S}}}^{\textnormal{MV}}(Q_{\mathcal{X}},Q_{\mathcal{X}})$ and its empirical counterpart $\mathbb{E}_{\mathbb{S}\sim Q}\;dis_{\rho_{\mathbb{S}}}^{\textnormal{MV}}(\mathbb{S}_{\mathcal{X}},\mathbb{T}_{\mathcal{X}})$, according to a convex function $\Delta:[0,1]^{2}\rightarrow\mathbb{R}$. Note that, as the bound \ref{General PAC-Bayes bound goyal} we incorporate in our general theorem the posterior and hyper-posterior distribution $\mathcal{Q}_{v,\mathbb{S}}/\rho_{\mathbb{S}}$ outputted by a given learning algorithm after observing the learning sample $\mathbb{S}$

\begin{theorem}\label{our general theorem}
    $\forall \, v \in [\![\mathcal{V}]\!]$, for any set of voters $\mathcal{H}_{v}$ for any marginal distributions $Q_{\mathcal{X}}$ and $P_{\mathcal{X}}$ over $\mathcal{X}$, any set of posterior distribution $\{\mathcal{Q}_{v,\mathbb{S}}\}_{v=1}^{V}$ on $\mathcal{H}_{v}$, for any hyper-posterior distribution $\rho_{\mathbb{S}}$ over $[\![\mathcal{V}]\!]$, for any set of prior distributions $\{\mathcal{P}_{v}\}_{v=1}^{V}$ on $\mathcal{H}_{v}$, for any hyper-prior distribution $\pi$ over $[\![\mathcal{V}]\!]$, for any $\delta \in (0, 1]$, for any $m > 0$, for any convex function $\Delta : [0, 1]^{2} \rightarrow \mathbb{R}$, with probability at least $1-\delta$ over the choice of $(\mathbb{S}_{\mathcal{X}} \times \mathbb{T}_{\mathcal{X}}) \sim (Q_{\mathcal{X}} \times P_{\mathcal{X}} )^{m}$, with probability at least $1-\delta$, we have:
    \begin{equation}
        \begin{split}
            &\Delta\bigg(\underset{\mathbb{S}\sim Q^{m}}{\mathbb{E}}dis_{\rho_{\mathbb{S}}}^{MV}(\mathbb{S}_{\mathcal{X}},\mathbb{T}_{\mathcal{X}}),\underset{\mathbb{S}\sim Q^{m}}{\mathbb{E}}dis_{\rho_{\mathbb{S}}}^{MV}(Q_{\mathcal{X}},P_{\mathcal{X}})\bigg)\\ 
            \leq &
            \frac{2}{m}\Bigg[\underset{\mathbb{S}\sim Q^{m}}{\mathbb{E}}\,\underset{v\sim \rho_{\mathbb{S}}}{\mathbb{E}} D_{\textnormal{KL}}(\mathcal{Q}_{v,\mathbb{S}}\Vert \mathcal{P}_{v}) + \underset{\mathbb{S}\sim Q^{m}}{\mathbb{E}}\,D_{\textnormal{KL}}(\rho_{\mathbb{S}} \Vert \pi)\\ &+ \ln{\sqrt{\underset{\mathbb{S}\sim Q^{m}}{\mathbb{E}}\,\underset{(v,v')\sim \pi^{2}}{\mathbb{E}}\,\underset{(h,h')\sim \mathcal{P}_{v}^{2}}{\mathbb{E}}e^{m\Delta(\mathcal{R}_{\hat{d}}(\hat{h}),\mathcal{R}_{d}(\hat{h}))}}}\Bigg],
        \end{split}
    \end{equation}
\end{theorem}

where $\hat{h} = (h,h')$ a pair of hypothesis, $(x^{(v)}, x'^{(v)}) \sim (Q_{\mathcal{X}}\times P_{\mathcal{X}})^{m}$ a pair of examples, $\mathcal{L}_{d}(\hat{h},x^{(v)},x'^{(v)}) = |\mathcal{L}_{\monzeroun}(h(x'^{(v)}),h'(x'^{(v)})) - (h(x^{(v)}),h'(x^{(v)}))|$ and $ \mathcal{R}_{d}(\hat{h}) = \mathbb{E}_{(x^{(v)},x'^{(v)})\sim (Q_{\mathcal{X}}\times P_{\mathcal{X}})^{m}} \mathcal{L}_{d}(\hat{h},x^{(v)},x'^{(v)})$,\\ $\mathcal{R}_{\hat{d}}(\hat{h}) = \mathbb{E}_{(x^{(v)},x'^{(v)})\sim (\mathbb{S}_{\mathcal{X}}\times \mathbb{T}_{\mathcal{X}})^{m}} \mathcal{L}_{d}(\hat{h},x^{(v)},x'^{(v)})$ the risk of $\hat{h}$ on the joint distribution.

\begin{proof}
The proof uses the ideas of the techniques and tricks of Bégin et al. \cite{BEG16}, Theorem 4.  A detailed proof is available in additional materials.   
\end{proof}

\subsection{Specialization of multi-view domain disagreement to the Classical Approaches}

In this section, we provide specialization of our multiview theorem to the most popular PAC-Bayesian approaches. To do so, we follow the same principles as Germain et al. \cite{GER09,GER15a}. Selecting a well-suited deviation function $\Delta$ and by upper-bounding $\underset{\mathbb{S}\sim Q^{m}}{\mathbb{E}}\,\underset{(v,v')\sim \pi^{2}}{\mathbb{E}}\,\underset{(h,h')\sim \mathcal{P}_{v}^{2}}{\mathbb{E}}e^{m\Delta(\mathcal{R}_{\hat{d}}(\hat{h}),\mathcal{R}_{d}(\hat{h}))}$, we can derive easily the classical PAC-Bayesian theorems of \cite{MCA98a}, \cite{SEE02}, \cite{CAT07} presented in the section \ref{Simple Pac-Bayesian bounds}.First, we derive the specialization theorem \ref{our general theorem} to the McAllester \cite{MCA99b}’s point of view. 

\begin{corollary}\label{bound_hennequin_mcallester}

$\forall \, v \in [\![\mathcal{V}]\!]$, for any set of voters $\mathcal{H}_{v}$ for any marginal distributions $Q_{\mathcal{X}}$ and $P_{\mathcal{X}}$ over $\mathcal{X}$, any set of posterior distribution $\{\mathcal{Q}_{v,\mathbb{S}}\}_{v=1}^{V}$ on $\mathcal{H}_{v}$, for any hyper-posterior distribution $\rho_{\mathbb{S}}$ over $[\![\mathcal{V}]\!]$, for any set of prior distributions $\{\mathcal{P}_{v}\}_{v=1}^{V}$ on $\mathcal{H}_{v}$, for any hyper-prior distribution $\pi$ over $[\![\mathcal{V}]\!]$, for any $\delta \in (0, 1]$, with probability at least $1-\delta$, we have:

\begin{equation}
\begin{split}
&\bigg|\underset{\mathbb{S}\sim Q^{m}}{\mathbb{E}}dis_{\rho_{\mathbb{S}}}^{MV}(\mathbb{S}_{\mathcal{X}},\mathbb{T}_{\mathcal{X}}) - \underset{\mathbb{S}\sim Q^{m}}{\mathbb{E}}dis_{\rho_{\mathbb{S}}}^{MV}(Q_{\mathcal{X}},P_{\mathcal{X}})\bigg|\\ 
            \leq &
            \sqrt{\frac{1}{2m}\Bigg[\underset{\mathbb{S}\sim Q^{m}}{\mathbb{E}}\,\underset{v\sim \rho_{\mathbb{S}}}{\mathbb{E}} 2D_{\textnormal{KL}}(\mathcal{Q}_{v,\mathbb{S}}\Vert \mathcal{P}_{v}) + \underset{\mathbb{S}\sim Q^{m}}{\mathbb{E}}\,2D_{\textnormal{KL}}(\rho_{\mathbb{S}} \Vert \pi) + \ln{\frac{2\sqrt{m}}{\delta}}\Bigg]}.
    \end{split}
\end{equation}

\end{corollary}

\begin{proof}
The proof uses the ideas of the techniques and tricks of Germain et al. \cite{GER15,GER15a}.  A detailed proof is available in additional materials.   
\end{proof}

To derive a generalization bound with the Catoni \cite{CAT07}’s point of view—given a convex function $\mathcal{F}$ and a real number $c > 0$ we define the measure of deviation as $\Delta(a, b) = \mathcal{F}(b) - c$ (Germain et al., \cite{GER09,GER15,GER15a}). We obtain the following generalization bound:

\begin{corollary}\label{bound_hennequin_cat}
$\forall \, v \in [\![\mathcal{V}]\!]$, for any set of voters $\mathcal{H}_{v}$ for any marginal distributions $Q_{\mathcal{X}}$ and $P_{\mathcal{X}}$ over $\mathcal{X}$, any set of posterior distribution $\{\mathcal{Q}_{v,\mathbb{S}}\}_{v=1}^{V}$ on $\mathcal{H}_{v}$, for any hyper-posterior distribution $\rho_{\mathbb{S}}$ over $[\![\mathcal{V}]\!]$, for any set of prior distributions $\{\mathcal{P}_{v}\}_{v=1}^{V}$ on $\mathcal{H}_{v}$, for any hyper-prior distribution $\pi$ over $[\![\mathcal{V}]\!]$, for any $\delta \in (0, 1]$, $\forall \, \alpha > 0$, with probability at least $1-\delta$, we have:

\begin{equation}
\begin{split}
\underset{\mathbb{S}\sim Q^{m}}{\mathbb{E}}dis_{\rho_{\mathbb{S}}}^{MV}(Q_{\mathcal{X}},P_{\mathcal{X}}) \leq &\frac{2\alpha}{1-e^{-2\alpha}}\Biggr[\underset{\mathbb{S}\sim Q^{m}}{\mathbb{E}}dis_{\rho_{\mathbb{S}}}^{MV}(\mathbb{S}_{\mathcal{X}},\mathbb{T}_{\mathcal{X}})\\
& + \frac{\underset{\mathbb{S}\sim Q^{m}}{\mathbb{E}}\,\underset{v\sim \rho_{\mathbb{S}}}{\mathbb{E}}D_{\textnormal{KL}}(\mathcal{Q}_{v,\mathbb{S}}\Vert\mathcal{P}_{v}) + \underset{\mathbb{S}\sim Q^{m}}{\mathbb{E}}D_{\textnormal{KL}}(\rho_{\mathbb{S}}\Vert\pi) + \ln{\sqrt{\frac{1}{\delta}}}}{m\times \alpha}\Biggr].
    \end{split}
\end{equation}
\end{corollary}

\begin{proof}
The proof uses the ideas of the techniques and tricks of Germain et al. \cite{GER15,GER15a}.  A detailed proof is available in additional materials.   
\end{proof}

As stated in \cite{GER16}, to recover a PAC-Bayesian bound similar to that proposed by Seeger \cite{SEE02}; Langford \cite{LAN05}, we use as $\Delta$-function the Kullback-Leibler divergence:

\begin{corollary}\label{bound_hennequin_DKL}
    $\forall \, v \in [\![\mathcal{V}]\!]$, for any set of voters $\mathcal{H}_{v}$ for any marginal distributions $Q_{\mathcal{X}}$ and $P_{\mathcal{X}}$ over $\mathcal{X}$, any set of posterior distribution $\{\mathcal{Q}_{v,\mathbb{S}}\}_{v=1}^{V}$ on $\mathcal{H}_{v}$, for any hyper-posterior distribution $\rho_{\mathbb{S}}$ over $[\![\mathcal{V}]\!]$, for any set of prior distributions $\{\mathcal{P}_{v}\}_{v=1}^{V}$ on $\mathcal{H}_{v}$, for any hyper-prior distribution $\pi$ over $[\![\mathcal{V}]\!]$, for any $\delta \in (0, 1]$, with probability at least $1-\delta$, we have:
    \begin{equation}
    \begin{split}
    &D_{\textnormal{KL}}\bigg(\underset{\mathbb{S}\sim Q^{m}}{\mathbb{E}}dis_{\rho_{\mathbb{S}}}^{MV}(\mathbb{S}_{\mathcal{X}},\mathbb{T}_{\mathcal{X}}),\underset{\mathbb{S}\sim Q^{m}}{\mathbb{E}}dis_{\rho_{\mathbb{S}}}^{MV}(Q_{\mathcal{X}},P_{\mathcal{X}})\bigg)\\
    \leq &\frac{1}{m}\Bigg[\underset{\mathbb{S}\sim Q^{m}}{\mathbb{E}}\,\underset{v\sim \rho_{\mathbb{S}}}{\mathbb{E}}2D_{\textnormal{KL}}(\mathcal{Q}_{v,\mathbb{S}}\Vert\mathcal{P}_{v}) + \underset{\mathbb{S}\sim Q^{m}}{\mathbb{E}}2D_{\textnormal{KL}}(\rho_{\mathbb{S}}\Vert\pi) + \ln{\frac{2\sqrt{m}}{\delta}}\Bigg].
        \end{split}
    \end{equation}
\end{corollary}

\begin{proof}
The proof uses the ideas of the techniques and tricks of Germain et al. \cite{GER15,GER15a}.  A detailed proof is available in additional materials.   
\end{proof}

The bounds \ref{bound_hennequin_cat}, \ref{bound_hennequin_DKL}, \ref{bound_hennequin_mcallester} are demonstrated for m=n, i.e., the sizes of samples from source domain $\mathbb{S}/\mathbb{S}_{\mathcal{X}}$ are the same of the samples from target domain $\mathbb{T}/\mathbb{T}_{\mathcal{X}}$. The last result of this section tackles the situation where we assume $m \neq n$, i.e., the sizes of $\mathbb{S}/\mathbb{S}_{\mathcal{X}}$ and $\mathbb{T}/\mathbb{T}_{\mathcal{X}}$ are different.

\begin{corollary}\label{pac_bound_hennequin_size_different}

$\forall \, v \in [\![\mathcal{V}]\!]$, for any set of voters $\mathcal{H}_{v}$ for any marginal distributions $Q_{\mathcal{X}}$ and $P_{\mathcal{X}}$ over $\mathcal{X}$, any set of posterior distribution $\{\mathcal{Q}_{v,\mathbb{S}}\}_{v=1}^{V}$ on $\mathcal{H}_{v}$, for any hyper-posterior distribution $\rho_{\mathbb{S}}$ over $[\![\mathcal{V}]\!]$, for any set of prior distributions $\{\mathcal{P}_{v}\}_{v=1}^{V}$ on $\mathcal{H}_{v}$, for any hyper-prior distribution $\pi$ over $[\![\mathcal{V}]\!]$, for any $\delta \in (0, 1]$, with probability at least $1-\delta$, we have:

\begin{equation}
\begin{split}
        &\bigg|\underset{\mathbb{S}\sim Q^{m}}{\mathbb{E}}dis_{\rho_{\mathbb{S}}}^{\textnormal{MV}}(\mathbb{S}_{\mathcal{X}},\mathbb{T}_{\mathcal{X}}) - \underset{\mathbb{S}\sim Q^{m}}{\mathbb{E}}dis_{\rho_{\mathbb{S}}}^{\textnormal{MV}}(Q_{\mathcal{X}},P_{\mathcal{X}})\bigg|\\ 
            \leq &
            \sqrt{\frac{\underset{\mathbb{S}\sim Q^{m}}{\mathbb{E}}\,\underset{v\sim \rho_{\mathbb{S}}}{\mathbb{E}} 2D_{\textnormal{KL}}(\mathcal{Q}_{v,\mathbb{S}}\Vert \mathcal{P}_{v}) + \underset{\mathbb{S}\sim Q^{m}}{\mathbb{E}}\,2D_{\textnormal{KL}}(\rho_{\mathbb{S}} \Vert \pi) + \ln{\frac{2\sqrt{m}}{\delta}}}{2m}}\\
            & + \sqrt{\frac{\underset{\mathbb{S}\sim Q^{m}}{\mathbb{E}}\,\underset{v\sim \rho_{\mathbb{S}}}{\mathbb{E}} 2D_{\textnormal{KL}}(\mathcal{Q}_{v,\mathbb{S}}\Vert \mathcal{P}_{v}) + \underset{\mathbb{S}\sim Q^{m}}{\mathbb{E}}\,2D_{\textnormal{KL}}(\rho_{\mathbb{S}} \Vert \pi) + \ln{\frac{2\sqrt{n}}{\delta}}}{2n}}.
    \end{split}
\end{equation}

\end{corollary}

\begin{proof}
The proof uses the ideas of the techniques and tricks of Germain et al. \cite{GER15,GER15a}.  A detailed proof is available in additional materials.   
\end{proof}

\subsection{The PAC-Bayesian DA-Bound}

Finally, the Theorem \ref{MVDA_bound_hennequin} leads to a PAC-Bayesian bound based on both the empirical source error of the Gibbs classifier and the empirical Multi-view domain disagreement pseudometric estimated on a source and target samples. The following bound is based on Catoni’s approach \ref{bound_hennequin_cat}:

\begin{theorem}\label{theorem_hennequin_DA_Bound_pac_bayes}
    $\forall \, v \in [\![\mathcal{V}]\!]$, for any set of voters $\mathcal{H}_{v}$ for any marginal distributions $Q_{\mathcal{X}}$ and $P_{\mathcal{X}}$ over $\mathcal{X}$, any set of posterior distribution $\{\mathcal{Q}_{v,\mathbb{S}}\}_{v=1}^{V}$ on $\mathcal{H}_{v}$, for any hyper-posterior distribution $\rho_{\mathbb{S}}$ over $[\![\mathcal{V}]\!]$, for any set of prior distributions $\{\mathcal{P}_{v}\}_{v=1}^{V}$ on $\mathcal{H}_{v}$, for any hyper-prior distribution $\pi$ over $[\![\mathcal{V}]\!]$,  for any $\delta \in (0, 1]$, with probability at least $1-\delta$, we have:
    
    \begin{equation}
    \begin{split}
        \underset{\mathbb{S} \sim Q^{m}}{\mathbb{E}}\mathcal{R}_{P}(G_{\rho_{\mathbb{S}}}^{\textnormal{MV}})  \leq & \underset{\mathbb{S} \sim Q^{m}}{\mathbb{E}}c'\mathcal{R}_{\mathbb{S}}(G_{\rho_{\mathbb{S}}}^{\textnormal{MV}}) + \underset{\mathbb{S} \sim Q^{m}}{\mathbb{E}} \alpha'\frac{1}{2}dis_{\rho_{\mathbb{S}}}^{\textnormal{MV}}(\mathbb{S}_{\mathcal{X}},\mathbb{T}_{\mathcal{X}}) \\ &+ \Bigg(\frac{c'}{c} + \frac{\alpha'}{\alpha}\Bigg)\frac{\underset{\mathbb{S} \sim Q^{m}}{\mathbb{E}}\;D_{KL}(\mathcal{Q}_{v,\mathbb{S}}||\mathcal{P}_{v}) + \underset{\mathbb{S} \sim Q^{m}}{\mathbb{E}}\;D_{KL}(\rho_{v,\mathbb{S}}||\pi) + \ln{\frac{1}{\delta}}}{m}\\
        & + \lambda_{\rho} + \frac{1}{2}(\alpha'-1),
        \end{split}
    \end{equation}
\end{theorem}
where $c' = \frac{c}{1-e^{-c}}$ and $\alpha' = \frac{2\alpha}{1-e^{-2\alpha}}$.

\begin{proof}
In Theorem \ref{MVDA_bound_hennequin}, replace $\mathbb{E}_{\mathbb{S}\sim Q^{m}}\;\mathcal{R}_{Q}(G_{\rho_{\mathbb{S}}}^{\textnormal{MV}})$ and $\mathbb{E}_{\mathbb{S}\sim Q^{m}}\;dis_{\rho_{\mathbb{S}}}^{\textnormal{MV}}(Q_{\mathcal{X}},P_{\mathcal{X}})$ by their upper bound, obtained from Corollary 2 in \cite{GOYAL17} and Corollary \ref{bound_hennequin_cat}.
\end{proof}

\section{Discussions and Conclusion}
The primary contrast between our bounds \ref{bound_hennequin_cat}; \ref{bound_hennequin_DKL}, \ref{bound_hennequin_mcallester}, \ref{pac_bound_hennequin_size_different}, \ref{theorem_hennequin_DA_Bound_pac_bayes}, and Germain et al.'s bounds \cite{GER15} lies in the incorporation of view-specific prior and posterior distributions. This results in an extra term, $\mathbb{E}_{v\sim \rho}\;D_{\textnormal{KL}}(\mathcal{Q}_{v} \Vert \mathcal{P}_{v} )$, which represents the expected value of the view-specific Kullback-Leibler divergence term over the views$[\![\mathcal{V}]\!]$, based on the hyper-posterior distribution $\rho$. The second difference comes from the expectation over all the possible learning samples in bounds itself \cite{GOYAL17}. In this way, the expectation $\mathbb{E}_{\mathbb{S}\sim Q^{m}}$ is distributed for the all terms in the bounds. Thereby, the $D_{\textnormal{KL}}(\cdot \Vert \cdot)$ terms take account of the  the posterior and hyper-posterior distribution $\mathcal{Q}_{v,\mathbb{S}}/\rho_{\mathbb{S}}$ outputted by a given learning algorithm after observing the learning sample $\mathbb{S}$.

Finally in this paper, we propose a first PAC-Bayesian analysis of weighted majority vote classifiers for domain adaptation with the concept of multi-view learning. Our works is based on theoretical results and for the future we aim to derive from introduced bounds a new domain adaptation multi-view algorithm. We will build on the work of Germain et al. \cite{GER13} to propose a specialized algorithm for linear classifiers or to propose a specialized algorithm for neural networks \cite{SIC22}.   

\bibliographystyle{splncs04}
\bibliography{mybibliography}
\section{Mathematical Tools}
\begin{lemma}{(Markov’s inequality)}
    Let $Z$ be a random variable and $t \geq 0$, then,
    \begin{displaymath}
        \mathbb{P}(|Z| \geq t) \leq \mathbb{E}(|Z|)/t .
    \end{displaymath}
\end{lemma}

\begin{lemma}\label{Jensen’s inequality}{(Jensen’s inequality)}
Let $Z$ be an integrable real-valued random variable and $g(·)$ any function.\\ 
If g(·) is convex, then,
\begin{displaymath}
    g(\mathbb{E}[Z]) \leq \mathbb{E}[g(Z)].
\end{displaymath}
If g(·) is concave, then, 
\begin{displaymath}
    g(\mathbb{E}[Z]) \geq \mathbb{E}[g(Z)].
\end{displaymath}

\end{lemma}

\begin{lemma}\label{Pinsker’s inequality}{(Pinsker’s inequality)}
    Consider two Bernoulli distributions $p$,$q$, and $d_{\textnormal{TV}}(p,q)$ the total variation distance between $p$ and $q$. Then, $D_{\textnormal{TV}} = |p-q|$ and we have 
    \begin{displaymath}
    \begin{split}
    D_{\textnormal{TV}}(p,q) &\leq \sqrt{\frac{1}{2}D_{\textnormal{KL}}(p\Vert q)}\\
    |p-q| &\leq \sqrt{\frac{1}{2}D_{\textnormal{KL}}(p\Vert q)}\\
    2(|p-q|)^{2} &\leq D_{\textnormal{KL}}(p\Vert q)
        \end{split}
    \end{displaymath}
\end{lemma}

\begin{lemma}{(Change of measure inequality)}

For any set $\mathcal{H}$, for any distributions $\mathcal{P}$ and $\mathcal{Q}$ on $\mathcal{H}$, and for any measurable function $\phi : \mathcal{H} \rightarrow \mathbb{R}$, we have

\begin{displaymath}
    \underset{h\sim \mathcal{Q}}{\mathbb{E}} \phi(h) \leq D_{\textnormal{KL}}(\mathcal{Q} \Vert \mathcal{P}) + \ln{\underset{h\sim \mathcal{P}}{\mathbb{E}}e^{\phi(h)}}  
\end{displaymath}
\end{lemma}

\begin{lemma}\label{Change of measure inequality extended to multi-view learning}{(Change of measure inequality extended to multi-view learning)}

For any set of priors $\{\mathcal{P}_{v}\}_{v=1}^{V}$ and any set of posteriors $\{\mathcal{Q}_{v}\}_{v=1}^{V}$, for any hyper-prior distribution $\pi$ on views $\nu$ and hyper-posterior distribution $\rho$ on $\nu$, and for any measurable function $\phi : \mathcal{H}_{v}\rightarrow \mathbb{R}$, we have

\begin{displaymath}
    \underset{v\sim \rho}{\mathbb{E}} \, \underset{h\sim \mathcal{Q}}{\mathbb{E}} \phi(h) \leq \underset{v\sim \rho}{\mathbb{E}}D_{\textnormal{KL}}(\mathcal{Q}_{v} \Vert \mathcal{P}_{v}) + D_{\textnormal{KL}}(\rho \Vert \pi) + \ln{\bigg(\underset{v\sim \pi}{\mathbb{E}}\,\underset{h\sim \mathcal{P}_{v}}{\mathbb{E}}e^{\phi(h)}\bigg)}  
\end{displaymath}
\end{lemma}

\begin{lemma}\label{lemma_2dkl}
Given any set $\mathcal{H}$, and any distributions $\pi$ and $\rho$ on $\mathcal{H}$, let $\hat{\rho}$ and $\hat{\pi}$ two distributions over $\mathcal{H}^{2}$ such that $\hat{\rho}(h,h') = \rho(h)\rho(h')$ and $\hat{\pi}(h,h') = \pi(h)\pi(h')$. Then 
\begin{displaymath}
    D_{\textnormal{KL}}(\hat{\rho}\Vert\hat{\pi}) = 2D_{\textnormal{KL}}(\rho\Vert\pi) 
\end{displaymath}
\begin{proof}
\begin{displaymath}
\begin{split}
    D_{\textnormal{KL}}(\hat{\rho}\Vert\hat{\pi}) &= \underset{(h,h')\sim \rho^{2}}{\mathbb{E}} \ln{\frac{\rho(h)\rho(h')}{\pi(h)\pi(h')}}\\
    &=\underset{h\sim \rho}{\mathbb{E}}\ln{\frac{\rho(h)}{\pi(h)}} + \underset{h'\sim \rho}{\mathbb{E}}\ln{\frac{\rho(h')}{\pi(h')}}\\
    &= 2\underset{h\sim\rho}{\mathbb{E}}\ln{\frac{\rho(h)}{\pi(h)}}\\
    &=2D_{\textnormal{KL}}(\rho\Vert\pi)
    \end{split}
\end{displaymath}
\end{proof}
\end{lemma}

\begin{lemma}\label{Maurer’s lemma 1}{(Maurer \cite{MAU2004})}
Let $X = (X_{1},\cdots, X_{m})$ be a vector of i.i.d. random variables, $0 \leq X_{i} \leq 1$, with $\mathbb{E}\, X_{i} = \mu$. Denote $X' = (X'_{1},\cdots, X'_{m})$, where $X'_{i}$ is the unique Bernoulli ($\{0, 1\}$-valued) random variable with $\mathbb{E}\, X'_{i} = \mu.$ If $f : [0, 1]^{n} \rightarrow \mathbb{R}$ is convex, then
\begin{displaymath}
    \mathbb{E}[f(X)] \leq \mathbb{E}[f(X')]
\end{displaymath}
    
\end{lemma}

\begin{lemma}\label{Maurer’s lemma}{(from Inequalities (1) and (2) of Maurer \cite{MAU2004}}. Let $m \geq 8$, and $X = (X_{1},\cdots, X_{m})$ be a vector of i.i.d. random variables, $0 \leq X_{i} \leq 1$. Then
\begin{displaymath}
    \sqrt{m} \leq \mathbb{E}\Bigg[ mD_{\textnormal{KL}}\Bigg(\frac{1}{m}\sum_{i=1}^{n}X_{i}\Bigg|\Bigg|\mathbb{E}[X_{i}]\Bigg)\Bigg] \leq 2\sqrt{m}
\end{displaymath}
\end{lemma}
\section{Detailed Proof of Theorem }
\begin{theorem}
    For each view  $v \in [\![\mathcal{V}]\!]$, for any set of voters $\mathcal{H}_{v}$, for any set of posterior distribution $\{\mathcal{Q}_{v}\}_{v=1}^{V}$ on $\mathcal{H}_{v}$, any hyper-posterior distribution $\rho$ over $[\![\mathcal{V}]\!]$. We have
    \begin{equation}
        \mathcal{R}_{P}(G) \leq \mathcal{R}_{Q}(G) + \frac{1}{2} dis_{\rho}^{\textnormal{MV}}(Q_{\mathcal{X}},P_{\mathcal{X}}) + \lambda_{\rho},
    \end{equation}
\end{theorem}

where $\lambda_{\rho}$ is the deviation between the expected joint errors between pairs for voters ad pairs of views on the target and source domains, which is defined as \\$\lambda_{\rho} = \Big|e_{P}^{\text{MV}}(\rho) - e_{Q}^{\text{MV}}(\rho)\Big|$ and where \\ $e_{P}^{\text{MV}}(\rho) = \mathbb{E}_{(\bm{x},y)\sim P}\;\mathbb{E}_{(v,v^{'})\sim \rho^{2}}\;\mathbb{E}_{(h,h')\sim\mathbb{Q}^{2}_{v}}\;\big[\mathcal{L}_{\monzeroun}\big(h(x),y\big)\times\mathcal{L}_{\monzeroun}\big(h'(x),y\big)\big],\; e_{Q}^{\text{MV}}(\rho) = \mathbb{E}_{(\bm{x},y)\sim Q}\;\mathbb{E}_{(v,v^{'})\sim \rho^{2}}\;\mathbb{E}_{(h,h')\sim\mathcal{Q}^{2}_{v}}\;\big[\mathcal{L}_{\monzeroun}\big(h(x),y\big)\times\mathcal{L}_{\monzeroun}\big(h'(x),y\big)\big]$. 

\begin{proof}
The proof presented below borrows the straightforward proof technique of \cite{GER15a}. First, from equation $\mathcal{R}_{Q}(G_{\mathcal{Q}}) = \frac{1}{2}\,d_{Q_{\mathcal{X}}}(\mathcal{Q}) + e_{Q}(\mathcal{Q})$ \cite{LAC06}, we deduce in our multi-view setting $\mathcal{R}_{Q}(G_{\rho}^{\text{MV}}) = \frac{1}{2}\,d_{Q_{\mathcal{X}}}^{\text{MV}}(\rho) + e_{Q}^{\text{MV}}(\rho)$

\begin{equation}
    \begin{split}
    \mathcal{R}_{P}(G_{\rho}^{\text{MV}}) - \mathcal{R}_{Q}(G_{\rho}^{\text{MV}}) & =\Big( \frac{1}{2}\,d_{P_{\mathcal{X}}}^{\text{MV}}(\rho) + e_{P}^{\text{MV}}(\rho)\Big) - \Big(\frac{1}{2}\,d_{Q_{\mathcal{X}}}^{\text{MV}}(\rho) + e_{Q}^{\text{MV}}(\rho)\Big)
    \\
    & = \frac{1}{2}\,\Big(d_{P_{\mathcal{X}}}^{\text{MV}}(\rho)-d_{Q_{\mathcal{X}}}^{\text{MV}}(\rho)\Big) + \Big(e_{P_{\mathcal{X}}}^{\text{MV}}(\rho)-e_{Q_{\mathcal{X}}}^{\text{MV}}(\rho)\Big)\\
     & \leq \frac{1}{2}\,\Big|d_{P_{\mathcal{X}}}^{\text{MV}}(\rho)-d_{Q_{\mathcal{X}}}^{\text{MV}}(\rho)\Big| + \Big|e_{P_{\mathcal{X}}}^{\text{MV}}(\rho)-e_{Q_{\mathcal{X}}}^{\text{MV}}(\rho)\Big|\\
      & = \frac{1}{2}\,dis_{\rho}^{\text{MV}}(Q_{\mathcal{X}},P_{\mathcal{X}}) + \lambda_{\rho}.
    \end{split}
\end{equation}

\end{proof}

\section{Detailed Proof of Theorem 8}\label{genral_theorem}

\begin{theorem}\label{our general theorem}
    $\forall \, v \in [\![\mathcal{V}]\!]$, for any set of voters $\mathcal{H}_{v}$ for any marginal distributions $Q_{\mathcal{X}}$ and $P_{\mathcal{X}}$ over $\mathcal{X}$, any set of posterior distribution $\{\mathcal{Q}_{v,\mathbb{S}}\}_{v=1}^{V}$ on $\mathcal{H}_{v}$, for any hyper-posterior distribution $\rho_{\mathbb{S}}$ over $[\![\mathcal{V}]\!]$, for any set of prior distributions $\{\mathcal{P}_{v}\}_{v=1}^{V}$ on $\mathcal{H}_{v}$, for any hyper-prior distribution $\pi$ over $[\![\mathcal{V}]\!]$, for any $\delta \in (0, 1]$, for any $m > 0$, for any convex function $\Delta : [0, 1]^{2} \rightarrow \mathbb{R}$, with probability at least $1-\delta$ over the choice of $(\mathbb{S}_{\mathcal{X}} \times \mathbb{T}_{\mathcal{X}}) \sim (Q_{\mathcal{X}} \times P_{\mathcal{X}} )^{m}$, with probability at least $1-\delta$, we have:
    \begin{equation}
        \begin{split}
            &\Delta\bigg(\underset{\mathbb{S}\sim Q^{m}}{\mathbb{E}}dis_{\rho_{\mathbb{S}}}^{MV}(\mathbb{S}_{\mathcal{X}},\mathbb{T}_{\mathcal{X}}),\underset{\mathbb{S}\sim Q^{m}}{\mathbb{E}}dis_{\rho_{\mathbb{S}}}^{MV}(Q_{\mathcal{X}},P_{\mathcal{X}})\bigg)\\ 
            \leq &
            \frac{2}{m}\Bigg[\underset{\mathbb{S}\sim Q^{m}}{\mathbb{E}}\,\underset{v\sim \rho_{\mathbb{S}}}{\mathbb{E}} D_{\textnormal{KL}}(\mathcal{Q}_{v,\mathbb{S}}\Vert \mathcal{P}_{v}) + \underset{\mathbb{S}\sim Q^{m}}{\mathbb{E}}\,D_{\textnormal{KL}}(\rho_{\mathbb{S}} \Vert \pi)\\ &+ \ln{\sqrt{\underset{\mathbb{S}\sim Q^{m}}{\mathbb{E}}\,\underset{(v,v')\sim \pi^{2}}{\mathbb{E}}\,\underset{(h,h')\sim \mathcal{P}_{v}^{2}}{\mathbb{E}}e^{m\Delta(\mathcal{R}_{\hat{d}}(\hat{h}),\mathcal{R}_{d}(\hat{h}))}}}\Bigg],
        \end{split}
    \end{equation}
\end{theorem}

\begin{proof}
    First, We propose to upper-bound,
    \begin{displaymath}
    \begin{split}
        &d = \underset{(v,v')\sim \rho^{2}_{\mathbb{S}}}{\mathbb{E}}\,\underset{(h,h')\sim \mathcal{Q}^{2}_{\mathbb{S}},v}{\mathbb{E}}[\mathcal{R}_{P}(h,h')-\mathcal{R}_{Q}(h,h')],\\ \textnormal{and its empirical counterpart},\\
        &\widehat{d} = \underset{(v,v')\sim \rho^{2}_{\mathbb{S}}}{\mathbb{E}}\,\underset{(h,h')\sim \mathcal{Q}^{2}_{\mathbb{S}}}{\mathbb{E}}[\mathcal{R}_{\mathbb{T}_{\mathcal{X}}}(h,h')-\mathcal{R}_{\mathbb{S}_{\mathcal{X}}}(h,h')],
            \end{split}
    \end{displaymath}
    
    Let us define the loss of on a pair of hypothesis $\hat{h} = (h,h')$ and a pair of examples $(x^{v}, x'^{v}) \sim (Q_{\mathcal{X}}\times P_{\mathcal{X}})^{m}$ by
    \begin{displaymath}
        \mathcal{L}_{d}(\hat{h},x^{v},x'^{v}) = \Big|\mathcal{L}_{\monzeroun}(h(x'^{v}),h'(x'^{v})) - \mathcal{L}_{\monzeroun}(h(x^{v}),h'(x^{v}))\Big|,
    \end{displaymath}
    
    Therefore, the risk of $\hat{h}$ on the joint distribution is defined as
    \begin{displaymath}
    \begin{split}
    \mathcal{R}_{d}(\hat{h}) &= \underset{(x^{v},x'^{v})\sim (Q_{\mathcal{X}}\times P_{\mathcal{X}})^{m}}{\mathbb{E}} \mathcal{L}_{d}(\hat{h},x^{v},x'^{v}),\\
    \mathcal{R}_{\hat{d}}(\hat{h}) &= \underset{(x^{v},x'^{v})\sim (\mathbb{S}_{\mathcal{X}}\times \mathbb{T}_{\mathcal{X}})^{m}}{\mathbb{E}} \mathcal{L}_{\hat{d}}(\hat{h},x^{v},x'^{v}).
        \end{split}
    \end{displaymath}
    
    The error of the related Gibbs multi-view classifier of these two quantitie are 
    \begin{displaymath}
    \begin{split}
        & \mathcal{R}_{d}(G_{\mathcal{Q}}^{MV}) = 
        \underset{(v,v')\sim \mathcal{\rho}^{2}}{\mathbb{E}}\,\underset{(h,h')\sim \mathcal{Q}^{2}_{\mathbb{S}}}{\mathbb{E}} \mathcal{R}_{d}(\hat{h}),\\ &\mathcal{R}_{\widehat{d}}(G_{\mathcal{Q}}) = \underset{(v,v')\sim \mathcal{\rho}^{2}}{\mathbb{E}}\,\underset{(h,h')\sim \mathcal{Q}^{2}_{\mathbb{S}}}{\mathbb{E}} \mathcal{R}_{\widehat{d}}(\hat{h}).\\
        & \text{ It is easy to show that }
        d = \mathcal{R}_{d}(G_{\mathcal{Q}}^{MV}) \text{ and } \hat{d} = \mathcal{R}_{\hat{d}}(G_{\mathcal{Q}}).
            \end{split}
    \end{displaymath}
    
    Now consider any convex function $\Delta : [0,1]^{2} \rightarrow \mathbb{R}$ and any $\mathcal{Q}_{\mathbb{S}}$ posterior distribution outputted by a given learning algorithm after observing the learning sample $\mathbb{S}$, we have,
    
    \begin{equation}
    \begin{split}    &m\Delta\bigg(\underset{\mathbb{S}\sim Q^{m}}{\mathbb{E}}\,\mathcal{R}_{\hat{d}}(G_{\mathcal{Q}_{\mathbb{S}_{v}}}^{MV}),\underset{\mathbb{S}\sim Q^{m}}{\mathbb{E}}\,\mathcal{R}_{d}(G_{\mathcal{Q}_{\mathbb{S}_{v}}}^{MV})\bigg)\\
    = & m\Delta\bigg(\underset{\mathbb{S}\sim Q^{m}}{\mathbb{E}}\,\underset{(v,v')\sim \rho^{2}} {\mathbb{E}}\,\underset{(h,h')\sim \mathcal{Q}^{2}_{v,\mathbb{S}}} {\mathbb{E}}\mathcal{R}_{\hat{d}}(h),\\ &\;\;\;\;\;\;\;\;\;\underset{\mathbb{S}\sim Q^{m}}{\mathbb{E}}\,\underset{(v,v')\sim \rho^{v}} {\mathbb{E}}\,\underset{(h,h')\sim \mathcal{Q}^{2}_{v,\mathbb{S}}} {\mathbb{E}}\mathcal{R}_{d}(h)\bigg)\\
    \leq &\underset{\mathbb{S}\sim Q^{m}}{\mathbb{E}}\,\underset{(v,v')\sim \rho^{v}} {\mathbb{E}}\,\underset{(h,h')\sim \mathcal{Q}^{2}_{v,\mathbb{S}}} {\mathbb{E}} m\Delta\big(\mathcal{R}_{\hat{d}}(h),\mathcal{R}_{d}(h)\big)\;(\textnormal{Jensen’s inequality} \ref{Jensen’s inequality})\\
    \leq & \underset{\mathbb{S}\sim Q^{m}}{\mathbb{E}} \Bigg[\underset{(v,v')\sim \rho^{2}_{\mathbb{S}}} {\mathbb{E}}D_{\textnormal{KL}}(\mathcal{Q}_{v,\mathbb{S}}^{2} \Vert \mathcal{P}_{v}^{2}) + D_{\textnormal{KL}}(\rho^{2}_{\mathbb{S}} \Vert \pi^{2}) + \ln{\Bigg(\underset{(v,v')\sim \pi^{2}}{\mathbb{E}}\,\underset{(h,h')\sim \mathcal{P}_{v}^{2}}{\mathbb{E}}e^{\Delta(\mathcal{R}_{\hat{d}}(\hat{h}),\mathcal{R}_{d}(\hat{h}))}\Bigg)}\Bigg]\\&(\textnormal{Change of measure inequality} \ref{Change of measure inequality extended to multi-view learning})
    \\= & \underset{\mathbb{S}\sim Q^{m}}{\mathbb{E}}\Bigg[ \underset{v\sim \rho_{\mathbb{S}}} {\mathbb{E}}\,2D_{\textnormal{KL}}(\mathcal{Q}_{v,\mathbb{S}} \Vert \mathcal{P}_{v}) + 2D_{\textnormal{KL}}(\rho_{\mathbb{S}} \Vert \pi) + \ln{\Bigg(\underset{(v,v')\sim \pi^{2}}{\mathbb{E}}\,\underset{(h,h')\sim \mathcal{P}_{v}^{2}}{\mathbb{E}}e^{\Delta(\mathcal{R}_{\hat{d}}(\hat{h}),\mathcal{R}_{d}(\hat{h}))}\Bigg)}\\
    = & \underset{\mathbb{S}\sim Q^{m}}{\mathbb{E}}\,\frac{2}{m}\Bigg[ 
    \underset{v\sim \rho_{\mathbb{S}}} {\mathbb{E}}\,D_{\textnormal{KL}}(\mathcal{Q}_{v,\mathbb{S}} \Vert \mathcal{P}_{v}) + \,D_{\textnormal{KL}}(\rho_{\mathbb{S}} \Vert \pi) + \ln{\sqrt{\underset{(v,v')\sim \pi^{2}}{\mathbb{E}}\,\underset{(h,h')\sim \mathcal{P}_{v}^{2}}{\mathbb{E}}e^{\Delta(\mathcal{R}_{\hat{d}}(\hat{h}),\mathcal{R}_{d}(\hat{h}))}}}\Bigg].
    \end{split}
    \end{equation}
    by replacing $\mathcal{R}_{d}(G_{\mathcal{Q}_{\mathbb{S}}}^{\textnormal{MV}})$ and its counterpart $\mathcal{R}_{\hat{d}}(G_{\mathcal{Q}_{\mathbb{S}}}^{\textnormal{MV}})$ by $d$ and $\hat{d}$ in the last equation we obtain our general theorem.
\end{proof}

\section{Detailed Proof of Corollary 1}
\begin{corollary}\label{bound_hennequin_mcallester}
$\forall \, v \in [\![\mathcal{V}]\!]$, for any set of voters $\mathcal{H}_{v}$ for any marginal distributions $Q_{\mathcal{X}}$ and $P_{\mathcal{X}}$ over $\mathcal{X}$, any set of posterior distribution $\{\mathcal{Q}_{v,\mathbb{S}}\}_{v=1}^{V}$ on $\mathcal{H}_{v}$, for any hyper-posterior distribution $\rho_{\mathbb{S}}$ over $[\![\mathcal{V}]\!]$, for any set of prior distributions $\{\mathcal{P}_{v}\}_{v=1}^{V}$ on $\mathcal{H}_{v}$, for any hyper-prior distribution $\pi$ over $[\![\mathcal{V}]\!]$, for any $\delta \in (0, 1]$, with probability at least $1-\delta$, we have:

\begin{equation}
\begin{split}
&\bigg|\underset{\mathbb{S}\sim Q^{m}}{\mathbb{E}}dis_{\rho_{\mathbb{S}}}^{MV}(\mathbb{S}_{\mathcal{X}},\mathbb{T}_{\mathcal{X}}) - \underset{\mathbb{S}\sim Q^{m}}{\mathbb{E}}dis_{\rho_{\mathbb{S}}}^{MV}(Q_{\mathcal{X}},P_{\mathcal{X}})\bigg|\\ 
            \leq &
            \sqrt{\frac{1}{2m}\Bigg[\underset{\mathbb{S}\sim Q^{m}}{\mathbb{E}}\,\underset{v\sim \rho_{\mathbb{S}}}{\mathbb{E}} 2D_{\textnormal{KL}}(\mathcal{Q}_{v,\mathbb{S}}\Vert \mathcal{P}_{v}) + \underset{\mathbb{S}\sim Q^{m}}{\mathbb{E}}\,2D_{\textnormal{KL}}(\rho_{\mathbb{S}} \Vert \pi) + \ln{\frac{2\sqrt{m}}{\delta}}\Bigg]}.
    \end{split}
\end{equation}
\end{corollary}

\begin{proof}
    Refer to the proof of Theorem \ref{genral_theorem}  for the definitions of $\mathcal{R}_{d}(\hat{h})$ and $\mathcal{R}_{d}(G_{\mathcal{Q}}^{MV})$, as well as their empirical counterparts $\mathcal{R}_{\hat{d}}(\hat{h})$ and $\mathcal{R}_{\hat{h}}(G_{\mathcal{Q}}^{MV})$. we apply Theorem\ref{genral_theorem} with $\Delta(a,b) = 2(|a-b|)^{2}$ Then, $\forall \delta \in (0,1]$, we obtain:
    \begin{displaymath}
    \begin{split}
        &2\bigg(\bigg|\underset{\mathbb{S}\sim Q^{m}}{\mathbb{E}}dis_{\rho_{\mathbb{S}}}^{MV}(\mathbb{S}_{\mathcal{X}},\mathbb{T}_{\mathcal{X}}) - \underset{\mathbb{S}\sim Q^{m}}{\mathbb{E}}dis_{\rho_{\mathbb{S}}}^{MV}(Q_{\mathcal{X}},P_{\mathcal{X}})\bigg|\bigg)^{2}\\ 
            \leq &
            \frac{1}{m}\Bigg[\underset{\mathbb{S}\sim Q^{m}}{\mathbb{E}}\,\underset{v\sim \rho_{\mathbb{S}}}{\mathbb{E}} 2D_{\textnormal{KL}}(\mathcal{Q}_{v,\mathbb{S}}\Vert \mathcal{P}_{v}) + \underset{\mathbb{S}\sim Q^{m}}{\mathbb{E}}\,2D_{\textnormal{KL}}(\rho_{\mathbb{S}} \Vert \pi)\\ &+ \ln{\bigg(\underset{\mathbb{S}\sim Q^{m}}{\mathbb{E}}\,\underset{(v,v')\sim \pi^{2}}{\mathbb{E}}\,\underset{(h,h')\sim \mathcal{P}_{v}^{2}}{\mathbb{E}}e^{m2(|\mathcal{R}_{\hat{d}}(\hat{h})-\mathcal{R}_{d}(\hat{h})|)^{2}}\bigg)}\Bigg].
        \end{split}
    \end{displaymath}
    Now, let us consider the non-negative random variable\\
    $\underset{S\sim Q^{m}}{\mathbb{E}}\,\underset{(v,v')\sim\pi^{2}}{\mathbb{E}}\,\underset{(h,h')\sim\mathcal{P}_{v}^{2}}{\mathbb{E}}e^{m2(|\mathcal{R}_{\hat{d}}(\hat{h})-\mathcal{R}_{d}(\hat{h})|)^{2}}$, $\mathcal{R}_{\hat{d}}(\hat{h})$ as a random variable which follows a binomial distribution of m trials with a probability of success $\mathcal{R}_{d}(\hat{h})$, we have:
    \begin{displaymath}
    \begin{split}
        \underset{S\sim Q^{m}}{\mathbb{E}}\,\underset{(v,v')\sim\pi^{2}}{\mathbb{E}}\,\underset{(h,h')\sim\mathcal{P}_{v}^{2}}{\mathbb{E}}e^{m2(|\mathcal{R}_{\hat{d}}(\hat{h})-\mathcal{R}_{d}(\hat{h})|)^{2}} &= \underset{(v,v')\sim\pi^{2}}{\mathbb{E}}\,\underset{(h,h')\sim\mathcal{P}_{v}^{2}}{\mathbb{E}}\,\underset{S\sim Q^{m}}{\mathbb{E}}e^{m2(|\mathcal{R}_{\hat{d}}(\hat{h})-\mathcal{R}_{d}(\hat{h})|)^{2}}\\&(\textnormal{change of the expectation})\\
        &  \leq \frac{1}{\delta}\underset{(v,v')\sim\pi^{2}}{\mathbb{E}}\,\underset{(h,h')\sim\mathcal{P}_{v}^{2}}{\mathbb{E}}\,\underset{S\sim Q^{m}}{\mathbb{E}}e^{m2(|\mathcal{R}_{\hat{d}}(\hat{h})-\mathcal{R}_{d}(\hat{h})|)^{2}}\\
        &\leq \frac{1}{\delta}\underset{(v,v')\sim\pi^{2}}{\mathbb{E}}\,\underset{(h,h')\sim\mathcal{P}_{v}^{2}}{\mathbb{E}}\,\underset{S\sim Q^{m}}{\mathbb{E}}e^{mD_{\textnormal{KL}}(\mathcal{R}_{\hat{d}}(\hat{h})\Vert\mathcal{R}_{d}(\hat{h}))}\\ &(\textnormal{Pinsker’s inequality Lemma \ref{Pinsker’s inequality}} )\\&
        \leq\frac{1}{\delta}\underset{(v,v')\sim\pi^{2}}{\mathbb{E}}\,\underset{(h,h')\sim\mathcal{P}_{v}^{2}}{\mathbb{E}}2\sqrt{m}\\&(\textnormal{Maurer’s lemma Lemma \ref{Maurer’s lemma}} )
        \\& = \frac{2\sqrt{m}}{\delta}.
        \end{split}
    \end{displaymath}
    By upper-bounding  $\underset{S\sim Q^{m}}{\mathbb{E}}\,\underset{(v,v')\sim\pi^{2}}{\mathbb{E}}\,\underset{(h,h')\sim\mathcal{P}_{v}^{2}}{\mathbb{E}}e^{m2(|\mathcal{R}_{\hat{d}}(\hat{h})-\mathcal{R}_{d}(\hat{h})|)^{2}}$ with $\frac{2\sqrt{m}}{\delta}$, we obtain the following inequality:
     \begin{displaymath}
    \begin{split}
        &\bigg|\underset{\mathbb{S}\sim Q^{m}}{\mathbb{E}}dis_{\rho_{\mathbb{S}}}^{MV}(\mathbb{S},\mathbb{T}) - \underset{\mathbb{S}\sim Q^{m}}{\mathbb{E}}dis_{\rho_{\mathbb{S}}}^{MV}(Q_{\mathcal{X}},P_{\mathcal{X}})\bigg|\\ 
            \leq &
            \sqrt{\frac{1}{2m}\Bigg[\underset{\mathbb{S}\sim Q^{m}}{\mathbb{E}}\,\underset{v\sim \rho_{\mathbb{S}}}{\mathbb{E}} 2D_{\textnormal{KL}}(\mathcal{Q}_{v,\mathbb{S}}\Vert \mathcal{P}_{v}) + \underset{\mathbb{S}\sim Q^{m}}{\mathbb{E}}\,2D_{\textnormal{KL}}(\rho_{\mathbb{S}} \Vert \pi) + \ln{\frac{2\sqrt{m}}{\delta}}\Bigg]}.
        \end{split}
    \end{displaymath}
\end{proof}

\section{Detailed Proof of Corollary 2}
\begin{corollary}
Let $V \geq 2$ be the number of views. For any distribution $D over \mathcal{X} \times \mathcal{Y}$, for any set of prior distributions $\{\mathcal{P}_{v}\}_{v=1}$, for any hyper-prior distribution $\pi$ over $V$,$\forall \, c > 0$ we have

\begin{equation}
\begin{split}
\underset{\mathbb{S}\sim Q^{m}}{\mathbb{E}}dis_{\rho_{\mathbb{S}}}^{MV}(Q_{\mathcal{X}},P_{\mathcal{X}}) \leq &\frac{2\alpha}{1-e^{-2\alpha}}\Biggr[\underset{\mathbb{S}\sim Q^{m}}{\mathbb{E}}dis_{\rho_{\mathbb{S}}}^{MV}(\mathbb{S},\mathbb{T})\\
& + \frac{\underset{\mathbb{S}\sim Q^{m}}{\mathbb{E}}\,\underset{v\sim \rho_{\mathbb{S}}}{\mathbb{E}}D_{\textnormal{KL}}(\mathcal{Q}_{v,\mathbb{S}}\Vert\mathcal{P}_{v}) + \underset{\mathbb{S}\sim Q^{m}}{\mathbb{E}}D_{\textnormal{KL}}(\rho_{\mathbb{S}}\Vert\pi) + \ln{\sqrt{\frac{1}{\delta}}}}{m\times \alpha}\Biggr].
    \end{split}
\end{equation}

\end{corollary}
\begin{proof}
    Refer to the proof of Theorem \ref{genral_theorem}  for the definitions of $\mathcal{R}_{d}(\hat{h})$ and $\mathcal{R}_{d}(G_{\mathcal{Q}}^{MV})$, as well as their empirical counterparts $\mathcal{R}_{\hat{d}}(\hat{h})$ and $\mathcal{R}_{\hat{h}}(G_{\mathcal{Q}}^{MV})$.
    As $\mathcal{L}_{d}$ defined in Theorem \ref{genral_theorem} lies in $[0, 1]$, we can bound $\mathcal{R}_{d}(G_{Q}^{\textnormal{MV}})$ following the proof process of Theorem 5 in \cite{GER15} (with $c = 2\alpha$). To do so, we take $\Delta(a,b) = \mathcal{F}(b)-a$ and we define the convex function, $\mathcal{F}(b) = -\ln{[1-(1-e^{-2\alpha})b]}$. We also consider the non-negative random variable $\underset{S\sim Q^{m}}{\mathbb{E}}\,\underset{(v,v')\sim\pi^{2}}{\mathbb{E}}\,\underset{(h,h')\sim\mathcal{P}_{v}^{2}}{\mathbb{E}}e^{m\mathcal{F}(\mathcal{R}_{d}(\hat{h}))-2\alpha\mathcal{R}_{\hat{d}}(\hat{h}))}$.\\
    $\forall \delta \in (0,1]$, we have,
    \begin{displaymath}
    \begin{split}
        \underset{S\sim Q^{m}}{\mathbb{E}}\,\underset{(v,v')\sim\pi^{2}}{\mathbb{E}}\,\underset{(h,h')\sim\mathcal{P}_{v}^{2}}{\mathbb{E}}e^{m\mathcal{F}(\mathcal{R}_{d}(\hat{h}))-2m\alpha\mathcal{R}_{\hat{d}}(\hat{h})} &\leq \frac{1}{\delta} \underset{(v,v')\sim\pi^{2}}{\mathbb{E}}\,\underset{(h,h')\sim\mathcal{P}_{v}^{2}}{\mathbb{E}}\, \underset{S\sim Q^{m}}{\mathbb{E}}e^{m\mathcal{F}(\mathcal{R}_{d}(\hat{h}))-2m\alpha \mathcal{R}_{\hat{d}}(\hat{h})},\\&(\textnormal{change of the expectation})\\
        &= \frac{1}{\delta}\underset{(v,v')\sim\pi^{2}}{\mathbb{E}}\Bigg[\underset{(h,h')\sim\mathcal{P}_{v}^{2}}{\mathbb{E}}e^{m\mathcal{F}(\mathcal{R}_{d}(\hat{h}))} \underset{S\sim Q^{m}}{\mathbb{E}}e^{-2m\alpha \mathcal{R}_{\hat{d}}(\hat{h})}\Bigg].
        \end{split}
    \end{displaymath}
    For a classifier $h$, let us define a random variable $X_{\hat{h}}$ that follows a binomial distribution of $m$ trials with a probability of success $\mathcal{R}_{d}(\hat{h})$ denoted by $B(m,\mathcal{R}_{d}(\hat{h}))$. We have:
    \begin{displaymath}
    \begin{split}
        \underset{S\sim Q^{m}}{\mathbb{E}}e^{-2m\alpha \mathcal{R}_{\hat{d}}(\hat{h})} &\leq \underset{X_{\hat{h}}\sim B(m,\mathcal{R}_{d}(\hat{h}))}{\mathbb{E}}e^{-2m\alpha X_{\hat{h}}}\\
        &(\textnormal{Maurer’s lemma Lemma \ref{Maurer’s lemma 1}})\\
        &=\sum_{k=0}^{m} \underset{X_{\hat{h}}\sim B(m,\mathcal{R}_{d}(\hat{h}))}{\mathbb{P}}(X_{\hat{h}} = k)e^{-2\alpha k}\\
        &=\sum_{k=0}^{m}\binom{m}{k}\mathcal{R}_{d}(\hat{h})^{k}(1-\mathcal{R}_{d}(\hat{h}))^{m-k}e^{-2\alpha k}\\
        &=\Big[\mathcal{R}_{d}(\hat{h})e^{-2\alpha} + (1-\mathcal{R}_{d}(\hat{h}))\Big]^{m}.
        \end{split}
    \end{displaymath}
    Then we obtain,
    \begin{displaymath}
        \begin{split}
           \underset{(h,h')\sim\mathcal{P}_{v}^{2}}{\mathbb{E}}e^{m\mathcal{F}(\mathcal{R}_{d}(\hat{h}))} \underset{S\sim Q^{m}}{\mathbb{E}}e^{-2m\alpha \mathcal{R}_{\hat{d}}(\hat{h})} \leq  &\underset{(h,h')\sim\mathcal{P}_{v}^{2}}{\mathbb{E}}e^{m\mathcal{F}(\mathcal{R}_{d}(\hat{h}))}\Big(\mathcal{R}_{d}(\hat{h})e^{-2\alpha} + (1-\mathcal{R}_{d}(\hat{h}))\Big)^{m}\\
           &=\underset{(h,h')\sim\mathcal{P}_{v}^{2}}{\mathbb{E}} 1 = 1.
        \end{split}
    \end{displaymath}
    We can now upper bound Equation simply by
    \begin{displaymath}
        \begin{split}
        \underset{S\sim Q^{m}}{\mathbb{E}}\,\underset{(v,v')\sim\pi^{2}}{\mathbb{E}}\,\underset{(h,h')\sim\mathcal{P}_{v}^{2}}{\mathbb{E}}e^{m\mathcal{F}(\mathcal{R}_{d}(\hat{h}))-2m\alpha\mathcal{R}_{\hat{d}}(\hat{h})} &\leq \frac{1}{\delta} 
        \end{split}
    \end{displaymath}
    \begin{displaymath}
        \begin{split}    
        \mathcal{F}\Big(\underset{\mathbb{S}\sim Q^{m}}{\mathbb{E}}dis_{\rho_{\mathbb{S}}}^{MV}(Q_{\mathcal{X}},P_{\mathcal{X}})\Big)&\leq \underset{\mathbb{S}\sim Q^{m}}{\mathbb{E}} 2\alpha dis_{\rho_{\mathbb{S}}}^{MV}(\mathbb{S}_{\mathcal{X}},\mathbb{T}_{\mathcal{X}})\Big) + \frac{1}{m}\Bigg[\underset{\mathbb{S}\sim Q^{m}}{\mathbb{E}}\,\underset{v\sim \rho_{\mathbb{S}}}{\mathbb{E}}2D_{\textnormal{KL}}(\mathcal{Q}_{v,\mathbb{S}}\Vert\mathcal{P}_{v})\\ &+ \underset{\mathbb{S}\sim Q^{m}}{\mathbb{E}}2D_{\textnormal{KL}}(\rho_{\mathbb{S}}\Vert\pi) + \ln{\frac{1}{\delta}}\Bigg]\\
        \underset{\mathbb{S}\sim Q^{m}}{\mathbb{E}} dis_{\rho_{\mathbb{S}}}^{MV}(Q_{\mathcal{X}},P_{\mathcal{X}}) &\leq \frac{1}{1-e^{-2\alpha}}\Bigg[ 1\\&-e^{-\big[\underset{\mathbb{S}\sim Q^{m}}{\mathbb{E}}2\alpha dis_{\rho_{\mathbb{S}}}^{MV}(\mathbb{S}_{\mathcal{X}},\mathbb{T}_{\mathcal{X}}) + \frac{1}{m}\big(\underset{\mathbb{S}\sim Q^{m}}{\mathbb{E}}\,\underset{v\sim \rho_{\mathbb{S}}}{\mathbb{E}}2D_{\textnormal{KL}}(\mathcal{Q}_{v,\mathbb{S}}\Vert\mathcal{P}_{v}) + \underset{\mathbb{S}\sim Q^{m}}{\mathbb{E}}2D_{\textnormal{KL}}(\rho_{\mathbb{S}}\Vert\pi) + \ln{\frac{1}{\delta}}\big)\big]}\Bigg]\\ &(\textnormal{from the inequality}\;1-e^{-x} \leq x)\\ \underset{\mathbb{S}\sim Q^{m}}{\mathbb{E}}dis_{\rho_{\mathbb{S}}}^{MV}(Q_{\mathcal{X}},P_{\mathcal{X}})\Big) &\leq \frac{1}{1-e^{-2\alpha}}\Bigg[ \underset{\mathbb{S}\sim Q^{m}}{\mathbb{E}} 2\alpha dis_{\rho_{\mathbb{S}}}^{MV}(\mathbb{S}_{\mathcal{X}},\mathbb{T}_{\mathcal{X}}) + \frac{1}{m}\Bigg(\underset{\mathbb{S}\sim Q^{m}}{\mathbb{E}}\,\underset{v\sim \rho_{\mathbb{S}}}{\mathbb{E}}2D_{\textnormal{KL}}(\mathcal{Q}_{v,\mathbb{S}}\Vert\mathcal{P}_{v})\\& + \underset{\mathbb{S}\sim Q^{m}}{\mathbb{E}}2D_{\textnormal{KL}}(\rho_{\mathbb{S}}\Vert\pi) + \ln{\frac{1}{\delta}}\Bigg)\Bigg]\\
        \underset{\mathbb{S}\sim Q^{m}}{\mathbb{E}}dis_{\rho_{\mathbb{S}}}^{MV}(Q_{\mathcal{X}},P_{\mathcal{X}}) &\leq \frac{2\alpha}{1-e^{-2\alpha}}\Biggr[\underset{\mathbb{S}\sim Q^{m}}{\mathbb{E}}dis_{\rho_{\mathbb{S}}}^{MV}(\mathbb{S}_{\mathcal{X}},\mathbb{T}_{\mathcal{X}})\\ &+ \frac{\underset{\mathbb{S}\sim Q^{m}}{\mathbb{E}}\,\underset{v\sim \rho_{\mathbb{S}}}{\mathbb{E}}D_{\textnormal{KL}}(\mathcal{Q}_{v,\mathbb{S}}\Vert\mathcal{P}_{v}) + \underset{\mathbb{S}\sim Q^{m}}{\mathbb{E}}D_{\textnormal{KL}}(\rho_{\mathbb{S}}\Vert\pi) + \ln{\sqrt{\frac{1}{\delta}}}}{m\times \alpha}\Biggr]
        \end{split}
    \end{displaymath}

\end{proof}
\section{Detailed Proof of Corollary 3}
\begin{corollary}\label{bound_hennequin_DKL}
    $\forall \, v \in [\![\mathcal{V}]\!]$, for any set of voters $\mathcal{H}_{v}$ for any marginal distributions $Q_{\mathcal{X}}$ and $P_{\mathcal{X}}$ over $\mathcal{X}$, any set of posterior distribution $\{\mathcal{Q}_{v,\mathbb{S}}\}_{v=1}^{V}$ on $\mathcal{H}_{v}$, for any hyper-posterior distribution $\rho_{\mathbb{S}}$ over $[\![\mathcal{V}]\!]$, for any set of prior distributions $\{\mathcal{P}_{v}\}_{v=1}^{V}$ on $\mathcal{H}_{v}$, for any hyper-prior distribution $\pi$ over $[\![\mathcal{V}]\!]$, for any $\delta \in (0, 1]$, with probability at least $1-\delta$, we have:
    \begin{equation}
    \begin{split}
    &D_{\textnormal{KL}}\bigg(\underset{\mathbb{S}\sim Q^{m}}{\mathbb{E}}dis_{\rho_{\mathbb{S}}}^{MV}(\mathbb{S}_{\mathcal{X}},\mathbb{T}_{\mathcal{X}}),\underset{\mathbb{S}\sim Q^{m}}{\mathbb{E}}dis_{\rho_{\mathbb{S}}}^{MV}(Q_{\mathcal{X}},P_{\mathcal{X}})\bigg)\\
    \leq &\frac{1}{m}\Bigg[\underset{\mathbb{S}\sim Q^{m}}{\mathbb{E}}\,\underset{v\sim \rho_{\mathbb{S}}}{\mathbb{E}}2D_{\textnormal{KL}}(\mathcal{Q}_{v,\mathbb{S}}\Vert\mathcal{P}_{v}) + \underset{\mathbb{S}\sim Q^{m}}{\mathbb{E}}2D_{\textnormal{KL}}(\rho_{\mathbb{S}}\Vert\pi) + \ln{\frac{2\sqrt{m}}{\delta}}\Bigg].
        \end{split}
    \end{equation}
\end{corollary}
\begin{proof}
    Refer to the proof of Theorem \ref{genral_theorem}  for the definitions of $\mathcal{R}_{d}(\hat{h})$ and $\mathcal{R}_{d}(G_{\mathcal{Q}}^{MV})$, as well as their empirical counterparts $\mathcal{R}_{\hat{d}}(\hat{h})$ and $\mathcal{R}_{\hat{d}}(G_{\mathcal{Q}}^{MV})$. we apply Theorem \ref{genral_theorem} with: 
    \begin{displaymath}
        D_{\textnormal{KL}}(q,p) = q\ln{\frac{q}{p}} + (1-q)\ln{\frac{1-q}{1-p}},
    \end{displaymath}
    Now, let us consider the non-negative random variable $\underset{S\sim Q^{m}}{\mathbb{E}}\,\underset{(v,v')\sim\pi^{2}}{\mathbb{E}}\,\underset{(h,h')\sim\mathcal{P}_{v}^{2}}{\mathbb{E}}e^{mD_{\textnormal{KL}}(\mathcal{R}_{\hat{d}}(\hat{h}),\mathcal{R}_{d}(\hat{h}))}$, $\mathcal{R}_{\hat{d}}(\hat{h})$ as a random variable which follows a binomial distribution of m trials with a probability of success $\mathcal{R}_{d}(\hat{h})$, we have:
    \begin{displaymath}
    \begin{split}
        \underset{S\sim Q^{m}}{\mathbb{E}}\,\underset{(v,v')\sim\pi^{2}}{\mathbb{E}}\,\underset{(h,h')\sim\mathcal{P}_{v}^{2}}{\mathbb{E}}e^{mD_{\textnormal{KL}}(\mathcal{R}_{\hat{d}}(\hat{h}),\mathcal{R}_{d}(\hat{h}))} &= \underset{(v,v')\sim\pi^{2}}{\mathbb{E}}\,\underset{(h,h')\sim\mathcal{P}_{v}^{2}}{\mathbb{E}}\,\underset{S\sim Q^{m}}{\mathbb{E}}e^{mD_{\textnormal{KL}}(\mathcal{R}_{\hat{d}}(\hat{h}),\mathcal{R}_{d}(\hat{h}))}\\& (\textnormal{change of the expectation})\\
        &  \leq \frac{1}{\delta}\underset{(v,v')\sim\pi^{2}}{\mathbb{E}}\,\underset{(h,h')\sim\mathcal{P}_{v}^{2}}{\mathbb{E}}\,\underset{S\sim Q^{m}}{\mathbb{E}}e^{mD_{\textnormal{KL}}(\mathcal{R}_{\hat{d}}(\hat{h}),\mathcal{R}_{d}(\hat{h}))}\\&
        \leq\frac{1}{\delta}\underset{(v,v')\sim\pi^{2}}{\mathbb{E}}\,\underset{(h,h')\sim\mathcal{P}_{v}^{2}}{\mathbb{E}}2\sqrt{m}
        \\& = \frac{2\sqrt{m}}{\delta}.
        \end{split}
    \end{displaymath}
    By upper-bounding  $\underset{S\sim Q^{m}}{\mathbb{E}}\,\underset{(v,v')\sim\pi^{2}}{\mathbb{E}}\,\underset{(h,h')\sim\mathcal{P}_{v}^{2}}{\mathbb{E}}e^{mD_{\textnormal{KL}}(\mathcal{R}_{\hat{d}}(\hat{h})\Vert\mathcal{R}_{d}(\hat{h}))}$ with $\frac{2\sqrt{m}}{\delta}$, we obtain the following inequality:
    \begin{displaymath}
    \begin{split}
    &D_{\textnormal{KL}}\bigg(\underset{\mathbb{S}\sim Q^{m}}{\mathbb{E}}dis_{\rho_{\mathbb{S}}}^{MV}(\mathbb{S}_{\mathcal{X}},\mathbb{T}_{\mathcal{X}}),\underset{\mathbb{S}\sim Q^{m}}{\mathbb{E}}dis_{\rho_{\mathbb{S}}}^{MV}(Q_{\mathcal{X}},P_{\mathcal{X}})\bigg)\\
    \leq &\frac{1}{m}\Bigg[\underset{\mathbb{S}\sim Q^{m}}{\mathbb{E}}\,\underset{v\sim \rho_{\mathbb{S}}}{\mathbb{E}}2D_{\textnormal{KL}}(\mathcal{Q}_{v,\mathbb{S}}\Vert\mathcal{P}_{v}) + \underset{\mathbb{S}\sim Q^{m}}{\mathbb{E}}2D_{\textnormal{KL}}(\rho_{\mathbb{S}}\Vert\pi) + \ln{\frac{2\sqrt{m}}{\delta}}\Bigg].
        \end{split}
    \end{displaymath}
\end{proof}

\section{Detailed Proof of Corollary 4}
\begin{proof}
    let us consider the non-negative random variable\\ $\underset{\mathbb{S}\sim Q^{m}}{\mathbb{E}}\,\underset{(v,v')\sim\pi^{2}}{\mathbb{E}}\,\underset{(h,h')\sim\mathcal{P}_{v}^{2}}{\mathbb{E}}e^{2m(\mathcal{R}_{Q_{\mathcal{X}}}(h,h')-\mathcal{R}_{\mathbb{S}_{\mathcal{X}}}(h,h'))^{2}}$, $\mathcal{R}_{\mathbb{S}_{\mathcal{X}}}(h,h')$ as a random variable which follows a binomial distribution of m trials with a probability of success $\mathcal{R}_{Q_{\mathcal{X}}}(h,h')$, we have:
    \begin{displaymath}
    \begin{split}
        \underset{\mathbb{S}\sim Q^{m}}{\mathbb{E}}\,\underset{(v,v')\sim\pi^{2}}{\mathbb{E}}\,\underset{(h,h')\sim\mathcal{P}_{v}^{2}}{\mathbb{E}}e^{2m(\mathcal{R}_{Q_{\mathcal{X}}}(h,h')-\mathcal{R}_{\mathbb{S}_{\mathcal{X}}}(h,h'))^{2}} &= \underset{(v,v')\sim\pi^{2}}{\mathbb{E}}\,\underset{(h,h')\sim\mathcal{P}_{v}^{2}}{\mathbb{E}}\,\underset{S\sim Q^{m}}{\mathbb{E}}e^{2m(\mathcal{R}_{Q_{\mathcal{X}}}(h,h')-\mathcal{R}_{\mathbb{S}_{\mathcal{X}}}(h,h'))^{2}}\\ &(\textnormal{change of the expectation})\\
        &  \leq \frac{1}{\delta}\underset{(v,v')\sim\pi^{2}}{\mathbb{E}}\,\underset{(h,h')\sim\mathcal{P}_{v}^{2}}{\mathbb{E}}\,\underset{S\sim Q^{m}}{\mathbb{E}}e^{mD_{\textnormal{KL}}(\mathcal{R}_{Q_{\mathcal{X}}}(h,h')\Vert\mathcal{R}_{\mathbb{S}_{\mathcal{X}}}(h,h'))}\\ &(\textnormal{Pinsker’s inequality Lemma \ref{Pinsker’s inequality}} )\\&
        \leq\frac{1}{\delta}\underset{(v,v')\sim\pi^{2}}{\mathbb{E}}\,\underset{(h,h')\sim\mathcal{P}_{v}^{2}}{\mathbb{E}}2\sqrt{m}\\&(\textnormal{Maurer’s lemma Lemma \ref{Maurer’s lemma}} )
        \\& = \frac{2\sqrt{m}}{\delta}.
        \end{split}
    \end{displaymath}
    By taking the logarithm on each outermost side of the previous inequality, we obtain
    \begin{displaymath}
        \ln{\underset{S\sim Q^{m}}{\mathbb{E}}\,\underset{(v,v')\sim\pi^{2}}{\mathbb{E}}\,\underset{(h,h')\sim\mathcal{P}_{v}^{2}}{\mathbb{E}}e^{2m(\mathcal{R}_{Q_{\mathcal{X}}}(h,h')-\mathcal{R}_{\mathbb{S}_{\mathcal{X}}}(h,h'))^{2}}} \leq \ln{\frac{2\sqrt{m}}{\delta}}
    \end{displaymath}
    Let us now find a lower bound of the left side of the last equation by using the change of measure inequality (Lemma \ref{Change of measure inequality extended to multi-view learning}), the Jensen inequality (Lemma \ref{Jensen’s inequality}) and Lemma \ref{lemma_2dkl}.
    \begin{displaymath}
    \begin{split}
         \ln{\underset{S\sim Q^{m}}{\mathbb{E}}\,\underset{(v,v')\sim\pi}{\mathbb{E}}\,\underset{(h,h')\sim\mathcal{P}_{v}}{\mathbb{E}}e^{2m(\mathcal{R}_{Q}(h,h')-\mathcal{R}_{\mathbb{S}}(h,h'))^{2}}} &\geq \\&\underset{S\sim Q^{m}}{\mathbb{E}}\,2m\Big(\underset{v\sim \rho}{\mathbb{E}} \, \underset{h\sim \mathcal{Q}}{\mathbb{E}}\;\mathcal{R}_{Q}(h,h')-\underset{v\sim \rho}{\mathbb{E}} \, \underset{h\sim \mathcal{Q}}{\mathbb{E}}\;\mathcal{R}_{\mathbb{S}}(h,h')\Big)^{2}\\&-\underset{S\sim Q^{m}}{\mathbb{E}}\;\underset{v\sim \rho}{\mathbb{E}}2D_{\textnormal{KL}}(\mathcal{Q}_{v,\mathbb{S}} \Vert \mathcal{P}_{v}) - \underset{S\sim Q^{m}}{\mathbb{E}}\;2D_{\textnormal{KL}}(\rho_{\mathbb{S}} \Vert \pi) \\ &\geq \underset{S\sim Q^{m}}{\mathbb{E}}\,2m\Big(\mathcal{R}_{Q}(G_{\rho_{\mathbb{S}}}^{\textnormal{MV}},G_{\rho_{\mathbb{S}}}^{\textnormal{MV}})-\mathcal{R}_{\mathbb{S}}(G_{\rho_{\mathbb{S}}}^{\textnormal{MV}},G_{\rho_{\mathbb{S}}}^{\textnormal{MV}})\Big)^{2}\\&-\underset{S\sim Q^{m}}{\mathbb{E}}\;\underset{v\sim \rho}{\mathbb{E}}2D_{\textnormal{KL}}(\mathcal{Q}_{v,\mathbb{S}} \Vert \mathcal{P}_{v}) - \underset{S\sim Q^{m}}{\mathbb{E}}\;2D_{\textnormal{KL}}(\rho_{\mathbb{S}} \Vert \pi)
         \end{split}
    \end{displaymath}
    We finally obtain,
    \begin{displaymath}
    \begin{split}
        \underset{S\sim Q^{m}}{\mathbb{E}}\,2m\Big(\mathcal{R}_{Q}(G_{\rho_{\mathbb{S}}}^{\textnormal{MV}},G_{\rho_{\mathbb{S}}}^{\textnormal{MV}})-\mathcal{R}_{\mathbb{S}}(G_{\rho_{\mathbb{S}}}^{\textnormal{MV}},G_{\rho_{\mathbb{S}}}^{\textnormal{MV}})\Big)^{2} \leq &\underset{S\sim Q^{m}}{\mathbb{E}}\;\underset{v\sim \rho}{\mathbb{E}}2D_{\textnormal{KL}}(\mathcal{Q}_{v,\mathbb{S}} \Vert \mathcal{P}_{v}) + \underset{S\sim Q^{m}}{\mathbb{E}}\;2D_{\textnormal{KL}}(\rho_{\mathbb{S}} \Vert \pi)\\ & 
 + \ln{\frac{2\sqrt{m}}{\delta}}
        \end{split}
    \end{displaymath}
        \begin{equation}
    \begin{split}
        \underset{S\sim Q^{m}}{\mathbb{E}}\,\Big|\mathcal{R}_{Q}(G_{\rho_{\mathbb{S}}}^{\textnormal{MV}},G_{\rho_{\mathbb{S}}}^{\textnormal{MV}})-\mathcal{R}_{\mathbb{S}}(G_{\rho_{\mathbb{S}}}^{\textnormal{MV}},G_{\rho_{\mathbb{S}}}^{\textnormal{MV}})\Big| \leq &\sqrt{\frac{1}{2m}\Bigg[\underset{S\sim Q^{m}}{\mathbb{E}}\;\underset{v\sim \rho}{\mathbb{E}}2D_{\textnormal{KL}}(\mathcal{Q}_{v,\mathbb{S}} \Vert \mathcal{P}_{v}) + \underset{S\sim Q^{m}}{\mathbb{E}}\;2D_{\textnormal{KL}}(\rho_{\mathbb{S}} \Vert \pi)}\\ & 
 + \overline{\ln{\frac{2\sqrt{m}}{\delta}}}\Bigg]
        \end{split}
    \end{equation}
    Following the exact same proof process with the random variable\\ $\underset{\mathbb{S}\sim Q^{m}}{\mathbb{E}}\,\underset{(v,v')\sim\pi^{2}}{\mathbb{E}}\,\underset{(h,h')\sim\mathcal{P}_{v}^{2}}{\mathbb{E}}e^{2n(\mathcal{R}_{P_{\mathcal{X}}}(h,h')-\mathcal{R}_{\mathbb{T}_{\mathcal{X}}}(h,h'))^{2}}$, we obtain,
    \begin{equation}
        \begin{split}
            \underset{\mathbb{S}\sim Q^{m}}{\mathbb{E}}\,\Big|\mathcal{R}_{P_{\mathcal{X}}}(G_{\rho_{\mathbb{S}}}^{\textnormal{MV}},G_{\rho_{\mathbb{S}}}^{\textnormal{MV}})-\mathcal{R}_{\mathbb{T}_{\mathcal{X}}}(G_{\rho_{\mathbb{S}}}^{\textnormal{MV}},G_{\rho_{\mathbb{S}}}^{\textnormal{MV}})\Big| \leq &\sqrt{\frac{1}{2n}\Bigg[\underset{\mathbb{S}\sim Q^{m}}{\mathbb{E}}\;\underset{v\sim \rho}{\mathbb{E}}2D_{\textnormal{KL}}(\mathcal{Q}_{v,\mathbb{S}} \Vert \mathcal{P}_{v}) + \underset{\mathbb{S}\sim Q^{m}}{\mathbb{E}}\;2D_{\textnormal{KL}}(\rho_{\mathbb{S}} \Vert \pi)}\\ & 
 + \overline{\ln{\frac{2\sqrt{n}}{\delta}}}\Bigg]
        \end{split}
    \end{equation}
    Joining Inequalities (7) and (8) with the union bound (that assure that both results hold simultaneously with probability $1-\delta$), gives the result because,
    \begin{displaymath}
        \begin{split}
             &\underset{S\sim Q^{m}}{\mathbb{E}}\,\Big|\mathcal{R}_{Q_{\mathcal{X}}}(G_{\rho_{\mathbb{S}}}^{\textnormal{MV}},G_{\rho_{\mathbb{S}}}^{\textnormal{MV}})-\mathcal{R}_{P_{\mathcal{X}}}(G_{\rho_{\mathbb{S}}}^{\textnormal{MV}},G_{\rho_{\mathbb{S}}}^{\textnormal{MV}})\Big| = dis_{\rho_{\mathbb{S}}}^{\textnormal{MV}}(Q_{\mathcal{X}},P_{\mathcal{X}}),\\ & \underset{S\sim Q^{m}}{\mathbb{E}}\,\Big|\mathcal{R}_{\mathbb{S}_{\mathcal{X}}}(G_{\rho_{\mathbb{S}}}^{\textnormal{MV}},G_{\rho_{\mathbb{S}}}^{\textnormal{MV}})-\mathcal{R}_{\mathbb{T}_{\mathcal{X}}}(G_{\rho_{\mathbb{S}}}^{\textnormal{MV}},G_{\rho_{\mathbb{S}}}^{\textnormal{MV}})\Big| = dis_{\rho_{\mathbb{S}}}^{\textnormal{MV}}(\mathbb{S}_{\mathcal{X}},\mathbb{T}_{\mathcal{X}}),
        \end{split}
    \end{displaymath}
    and because if $|a_1 - b_1| \leq c_1$ and $|a_2-b_2| \leq c_2$, then $|(a_1-a_2)- (b_1-b_2)| \leq c_1+c_2$.
\end{proof}
\end{document}